\documentclass{article} 
\usepackage[preprint]{neurips_2024}
\pdfoutput=1

\usepackage[utf8]{inputenc} 
\usepackage[T1]{fontenc}    
\usepackage{hyperref}       
\usepackage{url}            
\usepackage{booktabs}       
\usepackage{amsfonts}       
\usepackage{nicefrac}       
\usepackage{microtype}      
\usepackage{xcolor}         

\usepackage{amsmath, amssymb}
\usepackage{amsthm}
\usepackage{xspace} 
\usepackage{algorithm}
\usepackage{comment}
\usepackage{algpseudocode}

\usepackage{mathtools}
\usepackage{graphicx}
\usepackage{caption}
\usepackage{subcaption}
\usepackage{natbib}

\title{Efficient Reinforcement Learning in Probabilistic Reward Machines}

\author{
Xiaofeng Lin \\
Boston University\\
\texttt{xfl@bu.edu} \\
\And
Xuezhou Zhang\\
Boston University\\
\texttt{xuezhouz@bu.edu}\\
}

\newtheorem{definition}{Definition}
\newtheorem{lemma}{Lemma}
\newtheorem{theorem}{Theorem}

\usepackage{comment}
\usepackage{amsmath}
\usepackage{amssymb}
\usepackage{mathtools}
\usepackage{amsthm}

\usepackage{xspace}
\usepackage{booktabs}
\usepackage{algpseudocode}
\usepackage{subcaption}

\newcommand{\calA}{{\cal A}}

\newcommand{\calD}{{\cal D}}
\newcommand{\calE}{{\cal E}}
\newcommand{\calF}{{\cal F}}
\newcommand{\calG}{{\cal G}}
\newcommand{\calH}{{\cal H}}

\newcommand{\calK}{{\cal K}}

\newcommand{\calM}{{\cal M}}

\newcommand{\calO}{{\cal O}}
\newcommand{\calP}{{\cal P}}
\newcommand{\calQ}{{\cal Q}}
\newcommand{\calR}{{\cal R}}
\newcommand{\calS}{{\cal S}}

\newcommand{\mathbbE}{{\mathbb E}}
\newcommand{\mathbbP}{{\mathbb P}}
\newcommand{\mathbbR}{{\mathbb R}}
\newcommand{\mathbbV}{{\mathbb V}}
\newcommand{\mathbbW}{{\mathbb W}}

\newcommand{\indicator}{{\mathbb I}}
\newcommand{\Var}{\operatorname{Var}}
\newcommand{\ie}{\emph{i.e.},\xspace}
\newcommand{\eg}{\emph{e.g.},\xspace}
\newcommand{\defeq}{\stackrel{\text{def}}{=}}
\newcommand{\Regret}{{\text{Regret}}}
\newcommand{\alg}{\texttt{UCBVI-PRM}\xspace}
\newcommand{\Maxr}{G}
\newcommand{\delivery}{{\includegraphics[height=1.2em]{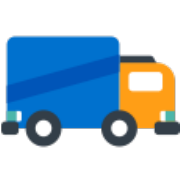}}\xspace}
\newcommand{\charging}{{\includegraphics[height=1.2em]{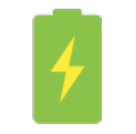}}\xspace}
\newcommand{\pickup}{{\includegraphics[height=1.2em]{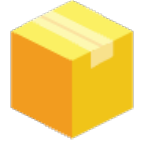}}\xspace}
\newcommand{\robot}{{\includegraphics[height=1.2em]{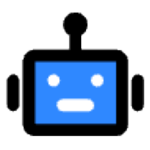}}\xspace}
\begin{document}

\maketitle

\begin{abstract}
In this paper, we study reinforcement learning in Markov Decision Processes with Probabilistic Reward Machines (PRMs), a form of non-Markovian reward commonly found in robotics tasks. We design an algorithm for PRMs that achieves a regret bound of $\widetilde{O}(\sqrt{HOAT} + H^2O^2A^{3/2} + H\sqrt{T})$, where $H$ is the time horizon, $O$ is the number of observations, $A$ is the number of actions, and $T$ is the number of time-steps. This result improves over the best-known bound, $\widetilde{O}(H\sqrt{OAT})$ of  \citet{pmlr-v206-bourel23a} for MDPs with Deterministic Reward Machines (DRMs), a special case of PRMs. When $T \geq H^3O^3A^2$ and $OA \geq H$, our regret bound leads to a regret of $\widetilde{O}(\sqrt{HOAT})$, which matches the established lower bound of $\Omega(\sqrt{HOAT})$ for MDPs with DRMs up to a logarithmic factor. To the best of our knowledge, this is the first efficient algorithm for PRMs. Additionally, we present a new simulation lemma for non-Markovian rewards, which enables reward-free exploration for any non-Markovian reward given access to an approximate planner.
Complementing our theoretical findings, we show through extensive experiment evaluations that our algorithm indeed outperforms prior methods in various PRM environments.
\end{abstract}

%

\section{Introduction}
Reinforcement learning traditionally focuses on the setting where the reward function is Markovian, meaning that it depends solely on the current state and action, and independent of history. However, in many real-world scenarios, the reward is a function of the history of states and actions. For example, consider a robot tasked with patrolling various locations in an industrial park. The performance of robot is measured by how thorough it regularly covers different zones in the park, which cannot easily be represented as a  function of its current state and action, but rather would depend on its whole trajectory during the patrol.

One emerging tool to model such performance metrics is called the Reward Machine (RM)\citep{pmlr-v80-icarte18a, icarte2022reward}, which is a Deterministic Finite-State Automaton (DFA) that can compress the sequence of past events into one single state. Combined with the current observation, the state of RM can fully specify the reward function. Hence, for an MDP with RM, we can obtain an equivalent cross-product MDP by leveraging the information of RM(see Lemma \ref{lemma_cross_product}) and applying off-the-shelf RL algorithms \eg Q-learning of \citet{sutton2018reinforcement} to learn an optimal policy. However, as we shall see later, this naive approach will incur a large regret.

One limitation of the classic RM framework is that the transition between the state of RM is restricted to be deterministic, whereas stochastic transitions are much more common in practice, especially with uncertainty in the environment. For instance, suppose a robot working in a warehouse is tasked with managing a warehouse by performing simple tasks of fetching and delivering items (as shown in Figure \ref{fig:warehouse_scenario}). The robot starts at a charging station, navigates to the item pickup location, collects the item, and then proceeds to the delivery location to deliver the item and receives a reward. Based on past experience and pre-collected data: there is a $20$ percent chance that the item at the pickup location is not ready, requiring the robot to wait until the item is ready, and a $10$ percent chance that the delivery location is occupied, causing the robot to wait before delivering the item. The robot is rewarded only when it successfully collects and delivers the item in sequence. The setting where the rewards can exhibit non-Markovian and stochastic dynamics can be represented as Probabilistic Reward Machine(PRM)\citep{dohmen2022inferring}.\\
\begin{figure*}[t]
    \centering
    \begin{subfigure}[b]{0.4\textwidth}
        \centering
        \includegraphics[width=0.6\textwidth]{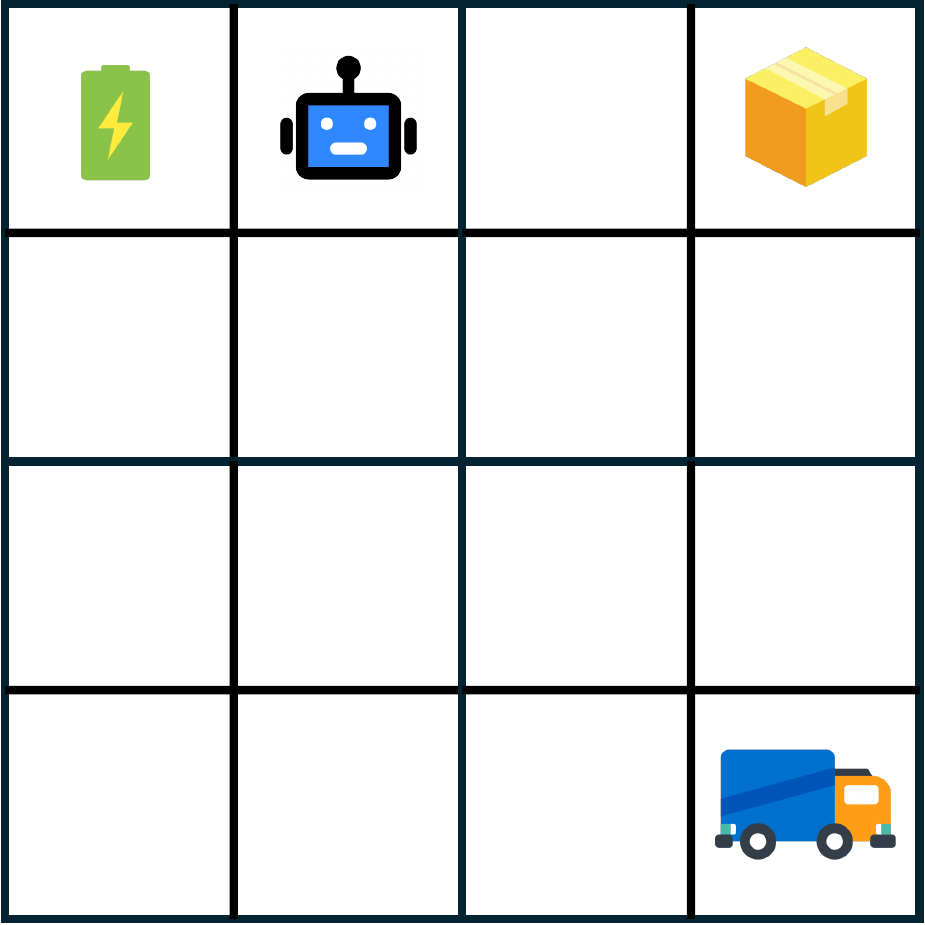}
        \caption{Warehouse environment}
        \label{fig:warehouse_env}
    \end{subfigure}
    \hspace{0.05\textwidth}
    \begin{subfigure}[b]{0.44\textwidth}
        \centering
        \includegraphics[width=\textwidth]{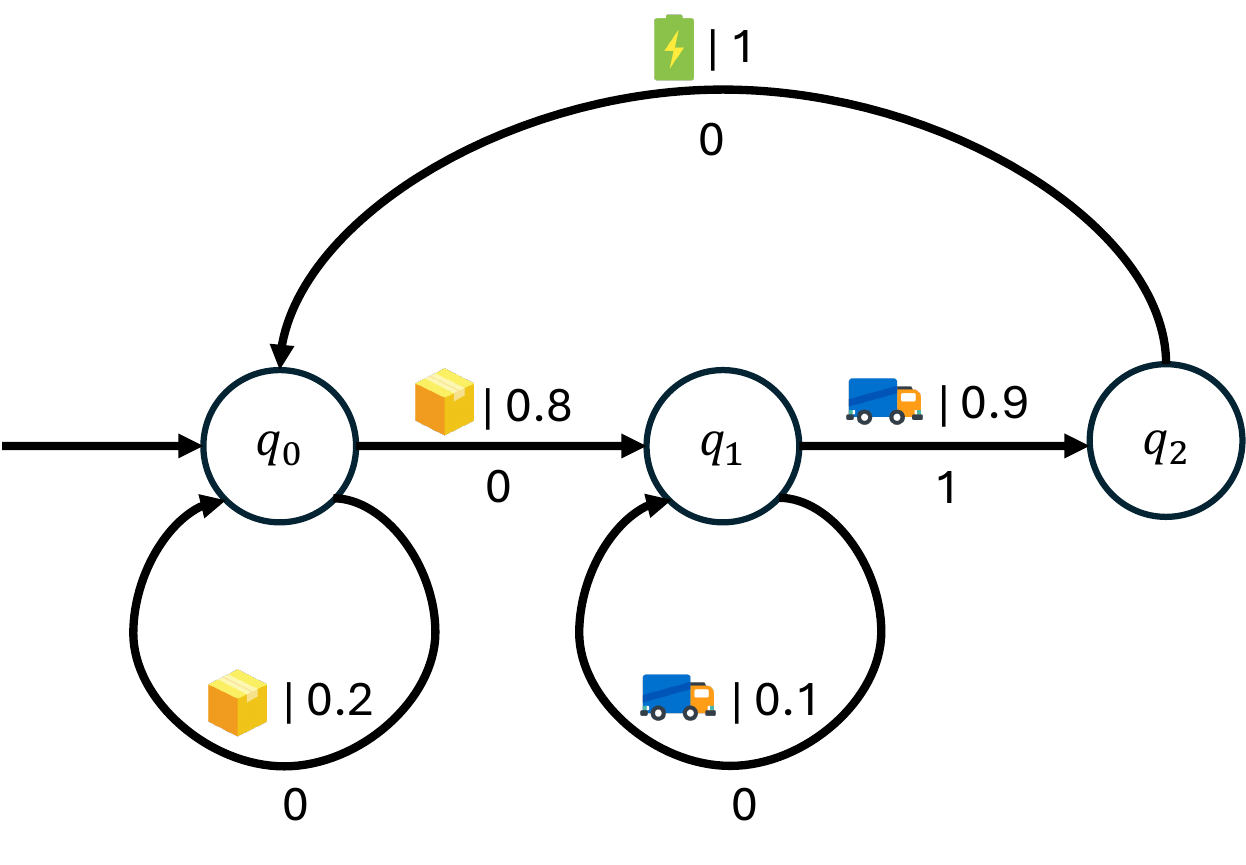}
        \caption{The Warehouse PRM}
        \label{fig:warehouse_prm}
    \end{subfigure}
    \caption{The Warehouse example and the corresponding PRM. The robot needs to pick up an item and delivers the item to the right location in sequence when the item may not be ready and the delivery location could be busy. \textbf{(a)}: a $4\times4$ grid world in which \robot is our robot, \charging is the charging station, \pickup is the item pickup location, \delivery is the delivery location;\textbf{(b)}: The corresponding PRM, where an edge $q \xrightarrow[r]{\ell \mid p} q'$ represents that state $q$ transitions to $q'$ on event $l$ with probability $p$ and receives reward $r$. }
    \label{fig:warehouse_scenario}
\end{figure*}
In this paper, we investigate RL in Markov decision processes with probabilistic reward machines. We formalize the regret minimization problem within the episodic MDP with PRM setting and introduce an algorithm, \texttt{UCBVI-PRM}, a UCB-style model-based RL algorithm with novel steps specifically tailored to PRMs. Our algorithm achieves a regret bound of $\widetilde{O}(\sqrt{HOAT})$ that matches the established lower bound of $\Omega(\sqrt{HOAT})$ for MDPs with PRMs up to a logarithmic factor. Additionally, we present a new simulation lemma that characterizes the difference in policy evaluations between two MDPs with generic non-Markovian rewards. Based on the lemma, we design a reward-free exploration algorithms that can collect data with sufficient coverage to learn a near-optimal policy under any non-Markovian reward in downstream tasks. Finally, we conduct experiments to showcase the efficiency of \alg.

\section{Related Work}
\paragraph{Reward Machines} \citet{pmlr-v80-icarte18a} introduced the concept of Deterministic Reward Machines (DRMs), a type of finite state machine that specifies reward functions while exposing the structure of these functions to the learner. They also proposed an algorithm called QRM, which outperforms standard Q-learning and hierarchical RL algorithms in environments where the reward can be specified by a DRM. Simultaneously, there has been extensive research on RL with temporal specifications expressed in Linear Temporal Logic (LTL) \citep{sadigh2014learning, wen2015correct, li2017reinforcement}, which can be directly translated into DRMs. Of particular interest to us is the recent work of \citet{pmlr-v206-bourel23a}, which studies online RL with DRMs in the context of infinite horizon average reward MDPs. They establish a regret lower-bound of $\Omega(\sqrt{HOAT})$ in the episodic setting. They also propose two algorithms differing in the confidence interval design, that achieve a regret bound of order $\widetilde{O}(H\sqrt{OAT})$ and $\widetilde{O}(HO\sqrt{AT})$, respectively. We improve upon this work with an algorithm that outperforms theirs both empirically and theoretically. 

Empirically, RL with LTL or DRMs has been successfully applied to complex robotics tasks, such as manipulation \citep{camacho2021reward} and locomotion \citep{defazio2024learning}. DRMs have also been employed in multi-agent learning scenarios \citep{neary2020reward, hu2024decentralized, zheng2024multi}. However, all of these works have focused exclusively on Deterministic RMs. \citet{dohmen2022inferring} introduced the concept of Probabilistic Reward Machines (PRMs), where state transitions are probabilistic rather than deterministic. However, to the best of our knowledge, no algorithms have yet been proposed to solve PRMs with rigorous theoretical guarantees.

\subsubsection{Reward-free exploration} 
Reward-free exploration \citep{jin2020rewardfree} studies the problem where an agent needs to collect data by interacting with the environment during the exploration stage, preparing for learning the optimal policy for an unknown reward function during the offline planning stage. The number of potential rewards can be arbitrarily large or even infinite. \citet{jin2020rewardfree} proposed an algorithm that can return an $\epsilon$-optimal policies for an arbitrary Markovian reward function, after at most collecting $\widetilde{O}\left(\frac{H^5S^2A}{\epsilon^2}\right)$ samples in the exploration stage. \citet{menard2021fast} further improved the sample complexity to $\widetilde{O}\left(\frac{H^3S^2A}{\epsilon^2}\right)$. Later work extends this problem to MDPs with more general structures, such as linear MDPs\citep{wang2020reward, zhang2021reward, wagenmaker2022reward} and kernel MDPs\citep{qiu2021reward}, and other settings such as multitask RL\cite{zhang2020task} and multi-agent RL\citep{bai2020provable, liu2021sharp}. However, \textit{all} of the above works restrict their attention on Markovian rewards. This is because their analysis crucially relies on the well-known simulation lemma that quantifies the value difference of a policy in two different MDPs. This classic version of simulation lemma only holds for Markovian reward. We provide the first extension of reward-free RL to the non-Markovian reward setting via a new simulation lemma that holds for arbitrary non-Markovian reward (Lemma 2).

\section{Problem Formulation}
We start with a few definitions.

\paragraph{Labeled MDPs with Probabilistic Reward Machines}
A labeled MDP\citep{xu2022joint} with PRM\citep{dohmen2022inferring} is defined as a tuple $M = (\mathcal{O}, \mathcal{A}, p, \mathcal{R}, \mathcal{P}, L, H)$. $\mathcal{O}$ represents a finite set of observations with cardinality $O$, and $\mathcal{A}$ represents a finite set of actions available at each observation with cardinality $A$. The transition function $p: \mathcal{O} \times \mathcal{A} \rightarrow \Delta_{\mathcal{O}}$ dictates the probability $p(o' | o, a)$ of transitioning to observation $o'$ when action $a$ is taken in observation $o$. The set $\calP$ consists of atomic propositions, and the labeling function $L: \calO \times \calA \times \calO \rightarrow 2^{\calP}$ assigns a subset of $\calP$ to each transition $(o, a, o')$. These labels represent high-level events associated with transitions that can be detected from the observation space. The component $\mathcal{R}$ is a Probabilistic Reward Machine (PRM), which generates history-dependent reward functions. A PRM is defined as a tuple $\calR = (\calQ, 2^{\calP}, \tau, \nu)$, where $\calQ$ is a finite set of states with cardinality $Q$. The probabilistic transition function $\tau: \calQ \times 2^{\calP} \rightarrow \Delta_{\calQ}$ determines the probability $\tau(q' | q, \sigma)$ of transitioning to state $q'$ given that the event $\sigma$ occurs in state $q$. The function $\nu: \mathcal{Q} \times 2^{\mathcal{P}} \times \mathcal{Q} \rightarrow \Delta_{[0,1]}$ maps each transition to a reward. The horizon $H$ defines the length of each episode. Note that MDP with PRM is a special class of Non-Markovian Reward Decision Processes (NMRDPs)\citep{bacchus1996rewarding}, which is identical to a MDP except that the reward depends on the history of state and action. For a NMRDP, we define $F(\eta)$ as the expected reward collected by trajectory $\eta \defeq \{(o_t, a_t)_{t=1}^H\}$ and $G\defeq \max_{\eta}F(\eta)$.

A special class of PRMs is Deterministic Reward Machines (DRMs)\citep{pmlr-v80-icarte18a, icarte2022reward}, where the transition function $\tau$ is deterministic. In a DRM, given a state $q$ and an event $\sigma$, the next state $q'$ is uniquely determined, and the rewards are generated deterministically based on these transitions.

The agent interacts with the MDP with PRM $M$ as follows: At each time step $t$, the agent is in observation $o_t \in \calO$ and selects an action $a_t \in \calA$ based on the history $h_t = (o_1, a_1, \ldots, o_{t-1}, a_{t-1}, o_t)$. After executing the action $a_t$ in observation $o_t$, the environment transitions to the next observation $o_{t+1} \sim p(\cdot | o_t, a_t)$ and assigns a label $\sigma_t = L(o_t, a_t, o_{t+1})$. Simultaneously, the PRM, which is in state $q_t$, transitions to state $q_{t+1} \sim \tau(\cdot | q_t, \sigma_t)$ and outputs a reward $\boldsymbol{r}_t = \nu(q_t, \sigma_t, q_{t+1})$. The agent then receives this reward, and both the environment and the PRM proceed to their next observation $o_{t+1}$ and state $q_{t+1}$, respectively. 

We define the joint state space as the cross-product of $\calQ$ and $\calO$, i.e. $\calS = \calQ \times \calO$. The transition function $P: \calS \times \calA \rightarrow \Delta_{\calS}$ dictates the probability $p(s' | s, a)$ of transitioning to state $s'$ when action $a$ is taken in $s$. The policy is defined as a mapping $\pi: \calS \times [H] \rightarrow \calA$. For every $\pi$, the occupancy measure $\mu^{\pi}_h: \calS \times \calA \rightarrow \Delta_{[0,1]}$ denotes the probability distribution at time $h$ induced by policy $\pi$ over state action space. Further we denote $\mu^{\pi}(s, a) \defeq \sum_{h=1}^H \mu_h^{\pi}(s, a)$. The value $V_h^{\pi}:\calS\rightarrow \mathbbR$ denotes the value function at each time step $h$ and state $s = (q, o) \in \calS$ is a tuple of $q \in \calQ$ and $o \in \calO$. $V_h^{\pi}(s)$ corresponds to the expected sum of $H-h$ rewards received under policy $\pi$, starting from $s_h = s\in \calS$. We define the state transition kernel $P_h^{\pi}$ under the policy $\pi$ at time step $h$ as $P_h^{\pi}$. $P_h^{\pi}(s'|s) \defeq P(s'|s, \pi(s, h)) = p(o'|o, \pi(s, h))\tau(q'|q, L(o, \pi(s, h), o'))$  where $s' = (q', o')$, and $r_h^{\pi} = R(s, \pi(s, h))$. For every $V:\calS \rightarrow \mathbbR$ the operators $P\cdot$ and $P_h^{\pi}\cdot$ are defined as $(PV)(s, a) \defeq \sum_{s'\in\calS} P(s'|s, a)V(s')$ for all $(s, a) \in \calS \times \calA$ and $(P_h^{\pi}V)(s) \defeq \sum_{s'\in\calS} P_h^{\pi}(s'|s)V(s')$.
We define the optimal value function $V_h^*(s) \defeq \sup_{\pi} V_h^{\pi}(s)$ for all $s \in \calS$ and $h \geq 1$. 
The learner’s goal is to minimize their cumulative regret given the knowledge of PRM, defined via
\begin{align*}
    \text{Regret}(K) := \sum_{k=1}^K \left( V_1^*(s_{k, 1}) - V_1^{\pi_k}(s_{k, 1})\right).    
\end{align*}
where $\pi_k$ the policy followed by the learner at episode $k$. The regret measures the loss of the learner who doesn't follow the optimal policy. Note that in most practical RL applications, human experts specify the reward functions to guide agents in learning the desired behaviors. Therefore, assuming knowledge of PRM is natural.

The following lemma makes it formal that the joint state transforms a PRM into an MDP with Markovian reward and transition.
\begin{lemma}\citep{pmlr-v206-bourel23a}\label{lemma_cross_product}
Let $M = (\calO, \calA, p, \calR,\calP, L, H)$ be a finite MDP with PRM. Then, an associated cross-product MDP to $M$ is $M_{cp} = (\calS, \mathcal{A}, P, R)$, where $\calS = \calQ \times \mathcal{O}$ and for $s = (q, o)$, $s' = (q', o') \in \calS$ and $a \in \mathcal{A}$,
\begin{align*}
&P(s' | s, a) = p(o' | o, a) \tau(q'|q, L(o, a, o'))\\
&R(s, a) = \sum_{o' \in \mathcal{O}, q' \in \mathcal{Q}} p(o' | o, a) \nu(q, L(o, a, o'), q').   
\end{align*}
\end{lemma}
This Lemma allows one to apply off-the-shelf RL algorithms to the cross-product MDP given the knowledge of PRM. However, the regret of such an approach will grow not slower than $\Omega(\sqrt{QOAHT})$\citep{auer2008near} that scales with the joint state space size $|Q|\cdot|O|$. In contrast, as we show next, it is possible to design algorithms whose regret scales only with the observation state space size $|O|$ and independent from the size of the label space, which can be as large as $|O|^H|A|^{H-1}$.

\section{Learning Algorithms and Results}

\begin{algorithm}[!ht]
\caption{\texttt{UCBVI-PRM}}\label{alg_UCBVIRM}
\textbf{Input:} Bonus algorithm $\texttt{bonus}$\\
\textbf{Initialize:} $D = \emptyset$
\begin{algorithmic}[1]
\For{episode $k \in [K]$}
    \State Compute, for all $(o, a, z) \in \calO\times\calA\times\calO$,
    \State $N_k(o, a, z) = \sum_{(o', a', z')\in \calH} \indicator{(o', a', z') = (o, a, z)}$
    \State $N_k(o, a) = \sum_{z\in\calO} N_k(o, a, z)$
    \State $N'_{k, h}(o, a) = \sum_{(o_{i, h}, a_{i, h})\in \calH}\indicator(o_{i, h} = o, a_{i, h} = a)$
    \State Let $\calK = \{(o, a)\in \calO\times \calA, N_k(o, a) \geq 0\}$
    \State Set $\widehat{p}_k(z|o, a) = \frac{N_k(o, a, z)}{N_k(o, a)}$ for all $(o, a) \in \calK$
    \For{timestep $h = H, H - 1, \ldots, 1$}
        \ForAll{$(o, a) \in \mathcal{O} \times A$}
            \If{$(o, a) \in \calK$}
                \ForAll{$s=(q,o),s'=(q', o'), q, q'\in \mathcal{Q}$}
                    \State 
                    \begin{align*}
                    &\widehat{P}_k(s' | s, a) = \widehat{p}_k(o'|o,a)\tau(q'|q, L(o,a,o'))\\
                    &\widehat{R}_k(s, a) = \sum_{s'\in \calS} \widehat{P}_k(s' | s, a)\nu(q, L(o,a,o'), q')
                    \end{align*}
                    \State 
                    \begin{align*}
                    &b_{k, h}(s, a) = \texttt{bonus}(\widehat{p}_k(o, a), V_{k, h+1}, N_k, N'_{k, h})
                    \end{align*}
                    \State  \begin{align*}
                        Q_{k, h}(s, a) = \min\{Q_{k-1, h}(s, a), H,\widehat{R}_k(s, a) + (\widehat{P}^k_h V_{k, h+1})(s, a) + b_{k, h}(s, a)\}
                    \end{align*}
                \EndFor
            \Else
                \State $Q_{k, h}(s, a) = H$
            \EndIf
        \EndFor
        \ForAll{$s \in S$}
            \State set $V_{k, h}(s) = \max_{a \in A} Q_{k,h}(s, a)$
        \EndFor
    \EndFor
    \For{timestep $h \in [H]$}
        \State take greedy action $a_{k, h} = \arg\max_{a \in A} Q_{k, h}(s_{k, h}, a)$
        \State observe state$(q_{k, h+1},o_{k, h+1})$
        \State update dataset $D \leftarrow D \cup \{ (h, o_{k, h}, a_{k, h}, o_{k, h+1}) \}$
    \EndFor
\EndFor
\end{algorithmic}
\end{algorithm}

In this section, we present our RL algorithm for PRMs, \texttt{UCBVI-PRM}. \texttt{UCBVI-PRM} follows the algorithmic skeleton of a classic model-based RL algorithm \citep{azar2017minimax}, while incorporating designs that leverage the structure of PRMs. Our key contribution is a regret bound of $\widetilde{O}(\sqrt{HOAT})$ when $T$ is large enough and $OA \geq H$. The regret bound matches the established lower bound $\Omega(\sqrt{HOAT})$ up to a logarithmic factor for MDP with DRM, and is notably independent of the joint state space size.

Intuitively, \texttt{UCBVI-PRM} (Algorithm \ref{alg_UCBVIRM}) proceeds in 3 stages:
(i) From lines 1 to 7, the algorithm first constructs an empirical transition matrix based on the data collected thus far; (ii) Using this empirical transition matrix, the algorithm then performs value iteration from lines 8 to 23 to update the value function.
Notably, between lines 8 and 19, the algorithm computes the new action-value function using the updated empirical transition matrix (line 12) and the exploration bonus (line 13); (iii) Finally, from lines 24 to 28, the agent selects actions based on the updated action-value function and collects new data, which is then incorporated into the dataset.

The main technical novelty lies in how we utilize the PRM structure. Denote $W_h: \calQ\times\calO\times\calA\times\calO\rightarrow \mathbbR$ a function that measures the expected return when being in state $(q, o)$, executing action $a$ at time step $h-1$ and observing $o'$ at time step $h$. $W$ is defined as follows:
\begin{align*}
     W_h(q,o,a,o') = \sum_{q'\in \calQ} \tau(q'|q, L(o, a, o'))V_h(q',o')
\end{align*}
$W_h$ is similar to value function $V_h$ in the sense that both are expected returns but condition on different random variables. Consequently, the estimation error can be characterized by $W_h$ instead of $V_h$, and our bonus will be a function of $W_h$ instead of $V_h$. More precisely, the estimation error $(\widehat{P}^{\pi_k}_k - P_h^{\pi_k})V^*_{h+1}$ can be translated to the estimation error in the observation space $(\widehat{p}^{\pi_k}_k - p^{\pi_k})W^*_{h+1}$. \\
\begin{algorithm}[t]
\caption{\texttt{bonus}}\label{alg_bonus}
\begin{algorithmic}[1]
\Require $\widehat{p}_k(o, a)$, $V_{k,h+1}$, $N_k$, $N'_{k,h}$
\ForAll{$o' \in \calO$}
\State $W_{k,h+1}(q,o,a,o') = \sum_{q'\in \calQ} \tau(q'|q, L(o, a, o'))V_{k, h+1}(q',o')$
\EndFor
\State $\widehat{\mathbbW}_{k, h+1} = \Var_{o'\sim \widehat{p}_k(\cdot|o, a)}\left(W_{k,h+1}(q,o,a,o')\right)$
\State $b_{k, h}(s, a) = \sqrt{\frac{8\iota \widehat{\mathbbW}_{k, h+1}}{N_k(o,a)}}+ \frac{14H\iota}{3N_k(o,a)} + \sqrt{\frac{2\iota}{N_k(o, a)}}+\sqrt{\frac{8 \sum_{o'\in \calO} \widehat{p}_k(o'|o,a)  \min \left( \frac{100^2 H^3 O^2 A\iota^2}{N'_{k,h+1}(o')}, H^2 \right) }{N_k(o,a)}}$
\State where $\iota = \ln(6QOAT/\rho)$
\State \Return $b_{k, h}(s, a)$
\end{algorithmic}
\end{algorithm}
We utilize a Bernstein-type reward bonus to ensure that \( V_{k,h} \) serves as an upper bound for \( V_h^* \) with high probability, a common technique in the literature \citep{azar2017minimax, zanette2019tighter, zhang2021reinforcement}. Unlike previous works that directly use \( V_{k,h} \), we leverage our knowledge of Probabilistic Reward Machines (PRMs) to construct our bonus using \( W_{k,h} \). This approach results in the regret associated with our bonus design growing only in the order of \( |O| \) instead of \( |S| \).

\begin{theorem}\label{thm:regret}(Regret bound for \texttt{UCBVI-PRM})
     Consider a parameter $\rho \in (0,1)$. Then the regret of \texttt{UCBVI-PRM} is bounded w.p. at least $1-\rho$, by
    \begin{align*}
        \Regret(K) \leq 72\iota\sqrt{OAHT} + 2500H^2O^2A^{3/2}\iota^2 + 4H\sqrt{T\iota}
    \end{align*}
    where $\iota = \ln(6QOAT/\rho)$.
\end{theorem}
Notice that the leading term of the regret does not scale with \( |Q| \). In contrast, if one were to apply an off-the-shelf RL algorithm to the cross-product MDP, it could achieve a regret bound no better than \( \Omega(\sqrt{HQOAT}) \) \citep{auer2008near}.
In the work of \citet{pmlr-v206-bourel23a}, their algorithms \ie \texttt{UCRL2-RM-L} and \texttt{UCRL2-RM-B} achieve a regret bound of \( \widetilde{O}(H\sqrt{OAT}) \) and \( \widetilde{O}(HO\sqrt{AT}) \) in DRMs, respectively. Compared to their results, we improve the regret bound by a factor of $\sqrt{H}$ and $\sqrt{HO}$ respectively, while generalizes to the PRM setting.

\section{RL with Non-Markoivian Rewards}\label{sec:exploration}
In this section, we introduce an explore-then-commit style algorithm for MDPs with generic non-Markovian rewards: in the exploration stage, the agent collects trajectories from the environment without the information of rewards; in the planning stage, the agent computes a near-optimal policy given the data gathered in the exploration stage, assuming access to an approximate planner. We give an efficient algorithm that conducts $\widetilde{O}(\frac{O^5A^3H^2\Maxr^2}{\epsilon^2})$ episodes of exploration and returns an $(\epsilon+\alpha)$-optimal policy for any general non-Markovian rewards, given an $\alpha$-approximate planner, formally stated below. 
\begin{definition}
    A planner is $\alpha$-approximate if given any NMRDP $M = (\calO, \calA, p, \calR, H)$, the planner returns a policy $\pi$ that satisfies
    \begin{align*}
        J(\pi) \geq J(\pi^*) - \alpha
    \end{align*}
    where $J(\pi)$ is the expected return of executing policy $\pi$ in $M$ and $\pi^*$ is the optimal policy in $M$. 
\end{definition}
\begin{theorem}\label{thm:reward_free}
The exists an absolute constant $c > 0$, such that, for any $p \in (0, 1)$, with probability at least $1 - \rho$, given the information collected by algorithm \ref{alg_reward_free}, algorithm \ref{AlgRFRLP} can output $(\epsilon+\alpha)$-optimal policies for any non-Markovian rewards assuming access to $\alpha$-approximate planner. The number of episodes in algorithm \ref{alg_reward_free} is bounded by
    \begin{align*}
        c\cdot\left[\frac{O^5A^3H^2\Maxr^2\iota'}{\epsilon^2} +\frac{O^4AH^4\Maxr\iota'^3}{\epsilon}\right]
    \end{align*}
    where $\iota' \defeq \ln{\left(\frac{OAH\Maxr}{\rho\epsilon}\right)}$.
\end{theorem}

This result is made possible by a new \textit{simulation lemma} that can be applied to generic non-Markovian rewards and non-Markovian policies.
\begin{lemma}\label{lemma_simulation}
For any two NMRDPs $M = (\calO, \calA, p, \calR, H)$ and $\widehat{M}= (\calO, \calA, \widehat{p}, \calR, H)$, for any policy $\pi$
\begin{align*}
    \left| \widehat{J}(\pi) - J(\pi) \right| \leq \sum_{m=1}^{H} \sum_{o_m, a_m, o_{m+1}} \epsilon(o_{m+1} | o_m, a_m) \mu_m^\pi(o_m, a_m) \Maxr
\end{align*}
where
\begin{align*}
    \epsilon(o_{m+1}|o_m, a_m) = |\widehat{p}(o_{m+1}|o_m, a_m) - p(o_{m+1}|o_m, a_m)|
\end{align*}
\end{lemma}
The proof can be found in the Appendix \ref{prf:lemma_simulation}. This lemma characterizes the performance difference of the same policy applied to two Non-Markovian Reward Decision Processes (NMRDPs) that differ in their transition kernels. The performance gap is determined by the divergence in the transition kernel $\epsilon$, the occupancy measure $\mu_h^{\pi}$ induced by the policy in $M$, and the return upperbound $\Maxr$.
This lemma is key to establish our result, because it can be applied to any non-Markovian reward and non-Markovian policy, including Markovian rewards and PRMs as special cases.
Intuitively, this lemma provides a concrete goal for the exploration stage, i.e. the gap $\epsilon(o_{m+1}|o_m, a_m)$ must be small for any $(o_m,a_m)$ pair that is visited significantly often under $\pi$. In the following, we show how to achieve this goal.

\subsection{Exploration stage}
\begin{algorithm}[tb]
\caption{Reward-free RL-Explore}\label{alg_reward_free}
\begin{algorithmic}[1]
\Require iteration number $N_0, N$.
\State set policy class $\Psi \leftarrow \emptyset$, and dataset $D \leftarrow \emptyset$.
\ForAll{$o \in \calO$}
    \State $r'(o', a') \leftarrow \mathbb{I}[o' = o]$ for all $(o', a') \in \calO \times \calA$.
    \State $\Phi_o \leftarrow \texttt{EULER}(r, N_0)$.
    \State $\pi(\cdot \mid o) \leftarrow \text{Uniform}(\calA)$ for all $\pi \in \Phi_{o}$.
    \State $\Psi \leftarrow \Psi \cup \Phi_{o}$.
\EndFor
\For{$n = 1 \text{ to } N$}
    \State sample policy $\pi \sim \text{Uniform}(\Psi)$.
    \State play $\mathcal{M}$ using policy $\pi$, and observe the trajectory $z_n = (o_1, a_1, \ldots, o_H, a_H, o_{H+1})$.
    \State $D \leftarrow D \cup \{z_n\}$
\EndFor
\State \textbf{return} dataset $D$.
\end{algorithmic}
\end{algorithm}
It turns out that a procedure (Algorithm \ref{alg_reward_free}) similar to the Markovian reward case suffices for our purpose \citep{jin2020rewardfree}. Intuitively, algorithm \ref{alg_reward_free} perform two steps. In the first step, from lines 2 to 7, the algorithm constructs a set of exploratory policies each designed to visit an observation state $o\in\calO$ as often as possible. To accomplish this, for each observation $o$, the algorithm creates a reward function that is 0 everywhere except at observation $o$, where it is set to 1 (line 3). The algorithm then employs a no-regret RL algorithm (e.g. \texttt{EULER} of \citet{zanette2019tighter}) to find a policy that maximizes this reward, which is equivalent to maximizing the probability of visiting $o$. In the second stage, from lines 8 to 12, the algorithm collects new data by sampling and executing policies from this exploratory policy set. We prove that, with this framework, the collected data can be used to learn a transition kernel that is sufficiently close to the true transition characterized by the divergence in Lemma 2. Towards this, we introduce the notion of significant observation:
\begin{definition}[Significant Observation]\label{def_sig_states}
A observation $o$ is $\delta$-significant if there exists a policy $\pi$, so that the probability to reach $o$ following policy $\pi$ is greater than $\delta$. In symbol:
\[
\max_\pi \mu^\pi(o) \geq \delta
\]
where $\mu^{\pi}(o) = \sum_{a}\mu^{\pi}(o, a)$.
\end{definition}
Intuitively, since insignificant observations are rarely reachable by any policy, their contributions to the divergence in Lemma 2 will be limited. Thus, it suffices to only visit significant observations. Algorithm \ref{alg_reward_free} is designed specifically for this purpose, and achieves the following guarantee.
\begin{theorem}\label{thm_sig} (Theorem 3 of \citep{jin2020rewardfree})
There exists absolute constant \( c > 0 \) such that for any \( \varepsilon > 0 \) and \( \rho \in (0, 1) \), if we set \( N_0 \geq cO^2AH^4\iota_0^3/\delta \) where \( \iota_0 = \log( OAH / \rho\delta) \), then with probability at least \( 1 - \rho \), that Algorithm \ref{alg_reward_free}  will return a dataset \( D \) consisting of \( N \) trajectories \( \{z_n\}_{n=1}^N \), which are i.i.d sampled from a distribution \( \lambda \) satisfying:
\begin{equation}\label{eq_SigStates}
\forall \delta-\text{significant}\ o, \quad \max_{a,\pi} \frac{\mu^{\pi}(o, a)}{\lambda(o, a)} \leq 2OAH.
\end{equation}
\end{theorem}
As we can see from theorem \ref{thm_sig}, all significant observations can be visited by distribution $\lambda$ with reasonable probability. Hence, with algorithm \ref{alg_reward_free}, we can learn our model by visiting significant observations without the guidance of any rewards.

\subsection{Planning stage}
\begin{algorithm}[h]
\caption{Rewards-free-Plan}\label{AlgRFRLP}
\begin{algorithmic}[1]
\State \textbf{Input:} a dataset of transition $D$, Rewards $\calR$.
\ForAll{$(o, a, o') \in \calO \times \calA \times \calO$}
    \State $N(o, a, o') \gets \sum_{(o_h,a_h,o_{h+1})\in D} \mathbb{I}[o_h = o, a_h = a, o_{h+1} = o']$.
    \State $N(o, a) \gets \sum_{o'} N(o, a, o')$.    
    \State $\widehat{p}(o' | o, a) = \frac{N(o, a, o')}{N(o, a)}$.    
    \State $\widehat{\pi} \gets \alpha\text{-approximate planner}(\widehat{p}, \calR)$.
\EndFor
\State \textbf{Return:} policy $\widehat{\pi}$.
\end{algorithmic}
\end{algorithm}
After collecting enough data in the exploration stage, algorithm \ref{AlgRFRLP} use the data to compute an empirical transition matrix $\widehat{p}$, on which the approximate planner is employed. We prove that (see Appendix \ref{prf:lemma_policy_diff}), any policy will have small value gap in the learned transition under $\widehat{p}$ vs. the ground truth transition $p$. 
\begin{lemma}\label{lemma_policy_diff}
There exists an absolute constant $c > 0$, for any $\epsilon > 0$, $\rho\in(0, 1)$, assume the dataset $\calD$ has $N$ i.i.d. samples from distribution $\lambda$ which satisfies equation \ref{eq_SigStates} with $\delta = \frac{\epsilon}{8O\Maxr}$, and $N \geq c\frac{O^5A^3H^2\Maxr^2\iota'}{\epsilon^2}$, where $\iota' = \ln{\left(\frac{2H}{\rho}\right)}$, then w.p. at least $1-2\rho$:
\begin{align*}
    \left| \widehat{J}(\pi) - J(\pi) \right| \leq \frac{\epsilon}{2}
\end{align*}
\end{lemma}
The reason for the increased sample complexity compared to the original analysis by \citet{jin2020rewardfree} lies in the fact that more samples are required to reduce the model error associated with significant observations than in the Markovian setting. Specifically, in our analysis, it is necessary to account for model errors across every $(o,a,o’)$ triplet. In contrast, in the standard Markovian setting, the modeling error can be further decomposed into the error of the empirical next-state value function (see the proof of Lemma 3.6 in \citet{jin2020rewardfree}), which allows for tighter bounds. After we obtain our empirical transition matrix $\widehat{p}$, given any non-Markovian rewards $\calR$, we can find a near optimal policy by running $\alpha$-approximate planner, as a result of our simulation Lemma. 
\begin{lemma}\label{lem:final_difference}
    Under the preconditions of lemma \ref{lemma_policy_diff}, with probability of $1-2\rho$ for any rewards $\calR$, the output policy of algorithm \ref{AlgRFRLP} is $\epsilon+\alpha$-optimal, that is 
    \begin{align*}
        \left| J(\pi^*)-J(\widehat{\pi})\right| \leq \epsilon+\alpha
    \end{align*}
    where $\pi^*$ is the optimal policy.
\end{lemma}

Note that, for general non-Markovian rewards, the optimization error won't be reduced to $0$, but for any PRMs, the optimization can be reduced to $0$, since we can run value iteration given the cross-product MDP and solve it optimally. In addition, there are some cases where the rewards possess special structural properties, for which performance with guarantees can be achieved\citep{prajapat2023submodular, de2024global}.
\section{Experiments}
\begin{figure*}[!ht]
    \centering
    \begin{subfigure}[b]{0.6\columnwidth}
        \centering
        \includegraphics[width=1.2\textwidth]{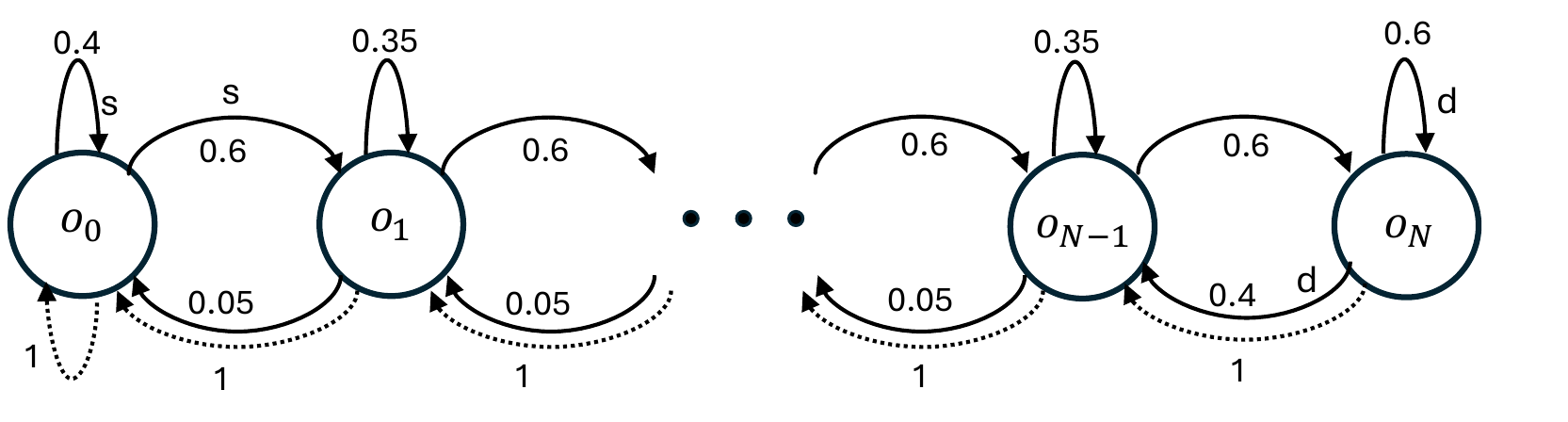}
        \caption{Labeled RiverSwim MDP}
        \label{fig:labeled_riverswim}
    \end{subfigure}
    \hspace{0.1\textwidth}
    \begin{subfigure}[b]{0.25\columnwidth}
        \centering
        \includegraphics[width=\textwidth]{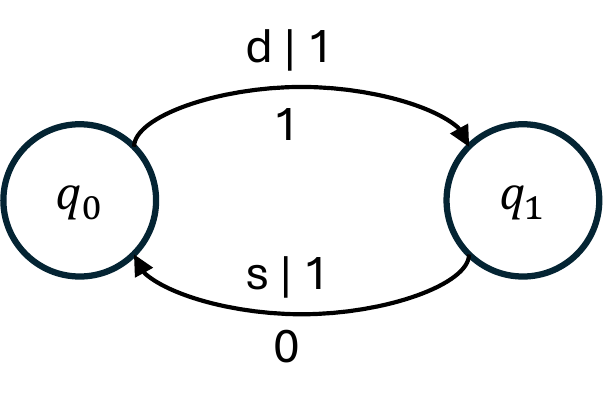}
        \caption{The Patrol DRM}
        \label{fig:riverswim_rm}
    \end{subfigure}
    \caption{The labeled RiverSwim and the corresponding DRM.}
    \label{fig:riverswim_rm_scenario}
\end{figure*}
In this section, we present a series of experiments comparing the empirical performance of our algorithm, \texttt{UCBVI-PRM}, with existing baselines. We evaluate our algorithm in MDPs with both DRM and PRM against different baselines. For DRMs, we compare with \texttt{UCRL-RM-L} and \texttt{UCRL-RM-B} of \citet{pmlr-v206-bourel23a}. For PRM, since there is no existing algorithm, we compare with the naive approach of directly applying \texttt{UCBVI}\citep{azar2017minimax} onto the cross-product MDP. 
In our experiment, we tune the exploration coefficient for all algorithms by selecting from a equally large set of options (see Appendix \ref{appendix:ec}). This is to make sure that an algorithm with a larger hyper-parameter set does not get an unfair advantage. In addition, we apply the doubling trick (detailed in Appendix \ref{appendix:dt}) to speed up \alg, which is a common technique in the literature\citep{auer2008near, dann2015sample} and won't affect the $\widetilde{O}$ regret.
\subsection{DRM Experiments}
In the RiverSwim environment, shown in Figure \ref{fig:riverswim_rm_scenario}, the agent has two actions corresponding to swimming left or right. Going right results in stochastic transitions, as shown by the solid lines in Figure \ref{fig:riverswim_rm_scenario}(a). Going left results in deterministic transitions as shown by the dashed lines in Figure \ref{fig:riverswim_rm_scenario}(a). The agent receives reward when visit two extreme locations in RiverSwim(\ie $o_1$ and $o_N$) in sequence.\\
Figure \ref{fig:rm_riverswim_results}(a), \ref{fig:rm_riverswim_results}(b), and \ref{fig:rm_riverswim_results}(c) show the regret over time in the RiverSwim domain, with the results averaged over 16 runs. The shaded area shows the standard variance of the corresponding quantity. Specifically, Figures \ref{fig:rm_riverswim_results}(a), \ref{fig:rm_riverswim_results}(b), and \ref{fig:rm_riverswim_results}(c) present the regrets of the agent running in a RiverSwim MDP with 5 observations and a horizon length of 10, a RiverSwim MDP with 10 observations and a horizon length of 20, and a RiverSwim MDP with 15 observations and a horizon length of 30, respectively. As we can see from the figures, in simpler environments (fewer observations and shorter horizons), the advantage of \alg is not obvious (see Figure \ref{fig:rm_riverswim_results}(a)). However, with longer horizons and more observations, the gap between \alg and the baselines of \citet{pmlr-v206-bourel23a} becomes larger. These results align with our regret analysis, where the regret of \alg grows slower than \texttt{UCRL-RM-L} in the order of $H$ and slower than \texttt{UCRL-RM-B} in the order of $H$ and $O$.
\begin{figure*}[!ht]
    \centering
    \begin{subfigure}[b]{0.32\columnwidth}
        \centering
        \includegraphics[width=\columnwidth]{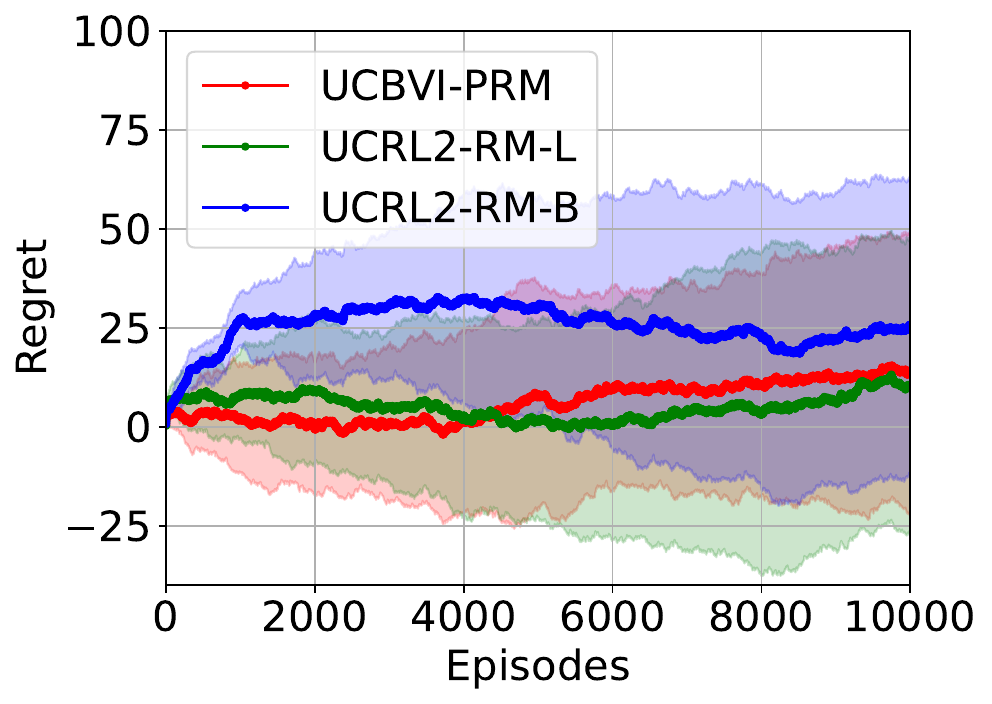}
        \caption{: $H=10$, $O=5$}
        \label{fig:rm_riverswim_regret_10h5s}
    \end{subfigure}
    \hfill
    \begin{subfigure}[b]{0.32\columnwidth}
        \centering
        \includegraphics[width=\columnwidth]{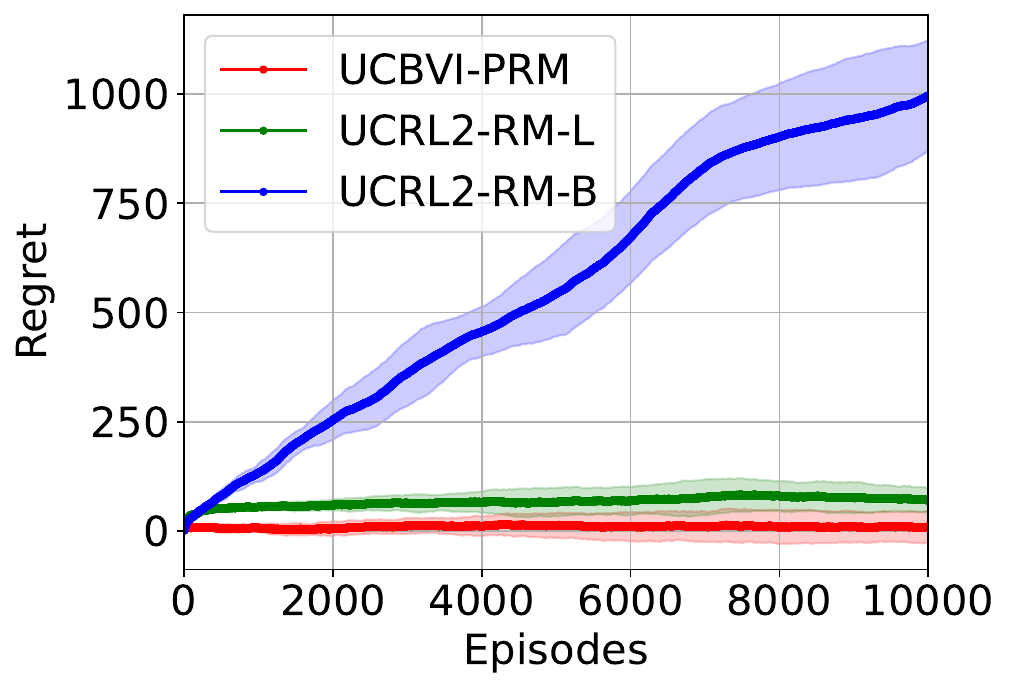}
        \caption{: $H=20$, $O=10$}
        \label{fig:rm_riverswim_regret_20h10s}
    \end{subfigure}
        \hfill
    \begin{subfigure}[b]{0.32\columnwidth}
        \centering
        \includegraphics[width=\columnwidth]{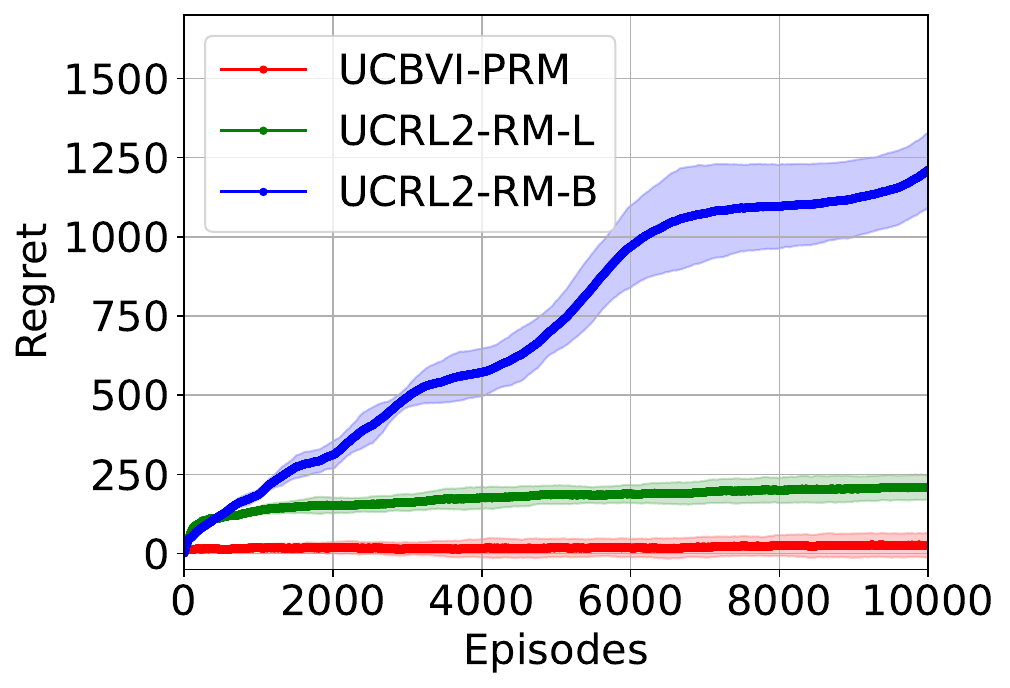}
        \caption{: $H=30$, $O=15$}
        \label{fig:rm_riverswim_regret_30h15s}
    \end{subfigure}
    \caption{Experimental results in RiverSwim}
    \label{fig:rm_riverswim_results}
\end{figure*}
\subsection{PRM Experiments}
In the warehouse environment(see Figure \ref{fig:warehouse_scenario}), the robot has five actions corresponding to moving up, right, down, left, and stay. Moving up, right, down, or left leads to moving in the intended direction with probability 0.7, in each perpendicular direction with probability 0.1, or staying in the same place with probability 0.1. The stay action will result in the robot staying in the same place deterministically. The robot receives reward when successfully picks up an item and delivers it to the delivery location in sequence.
\begin{figure*}[!ht]
    \centering
    \begin{subfigure}[b]{0.32\columnwidth}
        \centering
        \includegraphics[width=\columnwidth]{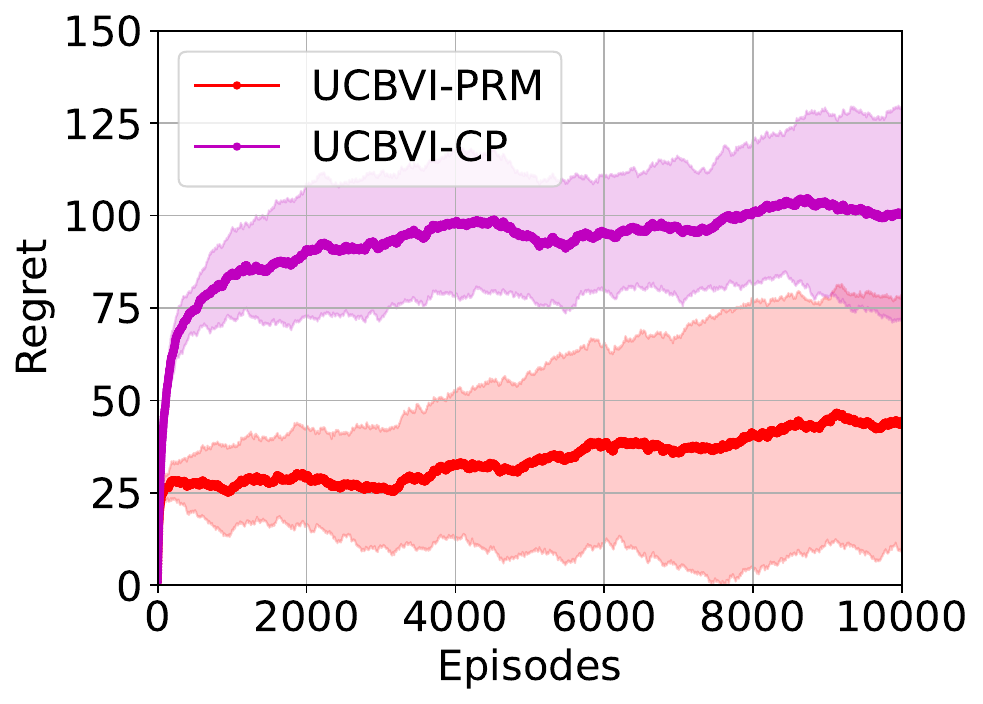}
        \caption{: $H=9$, $3\times3$ warehouse}
        \label{fig:warehouse_9h3s}
    \end{subfigure}
    \hfill
    \begin{subfigure}[b]{0.32\columnwidth}
        \centering
        \includegraphics[width=\columnwidth]{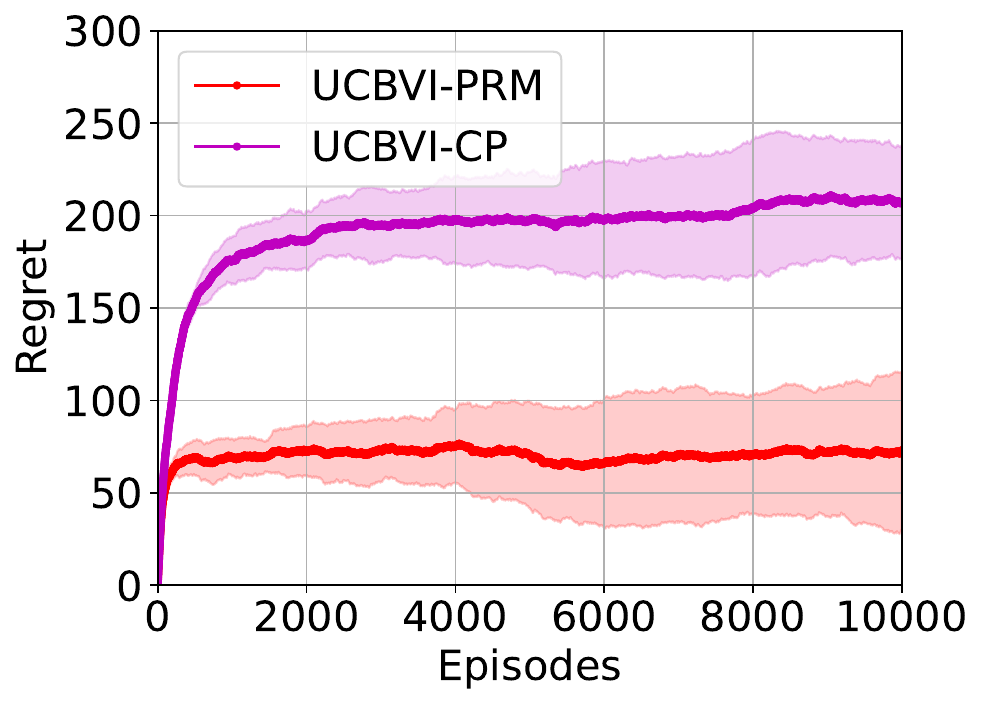}
        \caption{: $H=12$, $4\times4$ warehouse}
        \label{fig:warehouse_12h4s}
    \end{subfigure}
        \hfill
    \begin{subfigure}[b]{0.32\columnwidth}
        \centering
        \includegraphics[width=\columnwidth]{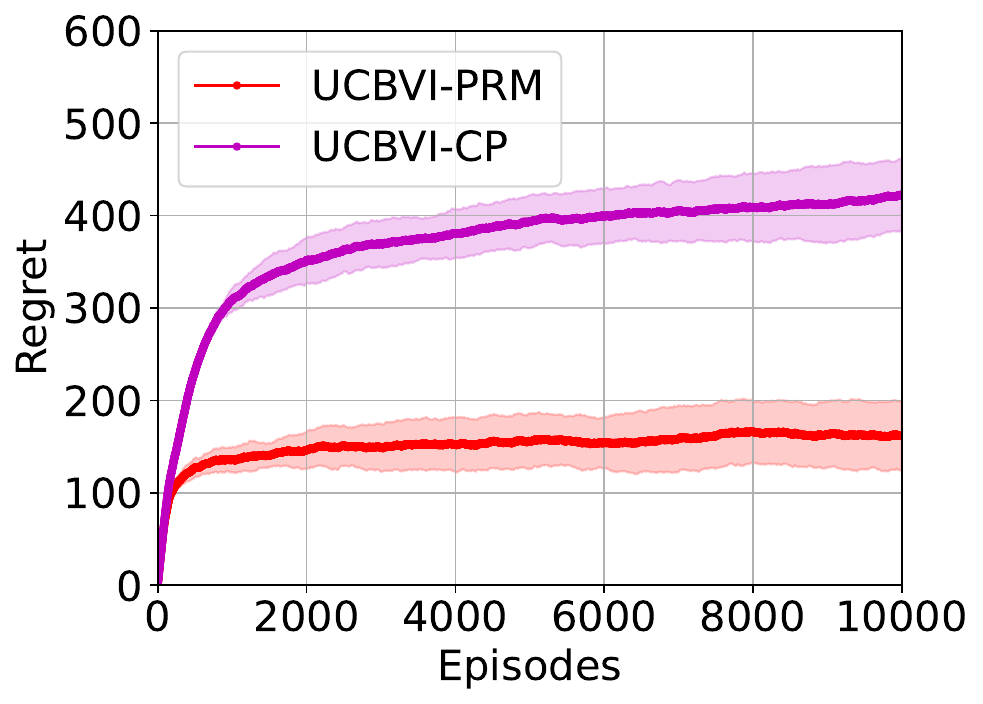}
        \caption{: $H=15$, $5\times5$ warehouse}
        \label{fig:warehouse_15h5s}
    \end{subfigure}
    \caption{Experimental results in Warehouse}
    \label{fig:warehouse_results}
\end{figure*}
Figures \ref{fig:warehouse_results} show the regret over time, with the results averaged over 16 runs. Specifically, Figures \ref{fig:warehouse_results}(a), \ref{fig:warehouse_results}(b), and \ref{fig:warehouse_results}(c) present the results of the agent running in a $3\times3$ warehouse with a horizon length of 9, a $4\times4$ warehouse with a horizon length of 12, and a $5\times5$ warehouse with a horizon length of 15, respectively. In all experiments, \alg outperforms the baseline. In addition, as the horizon becomes longer and with larger warehouse, \alg beats the baseline with a larger margin.
\section{Conclusion}
We study sample-efficient reinforcement learning in episodic Markov decision processes with probabilistic reward machines. We introduce an algorithm tailored for PRMs that matches the established lower bound when $T \geq H^3O^3A^2$ and $OA \geq H$. We also present a lemma that characterizes the difference in policy evaluations between two MDPs with non-Markovian rewards. Building upon the new lemma, we establish the reward-free learning result for non-Markovian reward. Finally, we validate our algorithms through a series of experiments. Interesting future direction includes designing effective and efficient algorithms for the multi-agent setting, and exploring connections with reward structures such as submodular rewards.
\bibliographystyle{plainnat}
\bibliography{references}

\begin{thebibliography}{36}
\providecommand{\natexlab}[1]{#1}
\providecommand{\url}[1]{\texttt{#1}}
\expandafter\ifx\csname urlstyle\endcsname\relax
  \providecommand{\doi}[1]{doi: #1}\else
  \providecommand{\doi}{doi: \begingroup \urlstyle{rm}\Url}\fi

\bibitem[Auer et~al.(2008)Auer, Jaksch, and Ortner]{auer2008near}
Peter Auer, Thomas Jaksch, and Ronald Ortner.
\newblock Near-optimal regret bounds for reinforcement learning.
\newblock \emph{Advances in neural information processing systems}, 21, 2008.

\bibitem[Azar et~al.(2017)Azar, Osband, and Munos]{azar2017minimax}
Mohammad~Gheshlaghi Azar, Ian Osband, and Rémi Munos.
\newblock Minimax regret bounds for reinforcement learning, 2017.

\bibitem[Bacchus et~al.(1996)Bacchus, Boutilier, and Grove]{bacchus1996rewarding}
Fahiem Bacchus, Craig Boutilier, and Adam Grove.
\newblock Rewarding behaviors.
\newblock In \emph{Proceedings of the National Conference on Artificial Intelligence}, pages 1160--1167, 1996.

\bibitem[Bai and Jin(2020)]{bai2020provable}
Yu~Bai and Chi Jin.
\newblock Provable self-play algorithms for competitive reinforcement learning.
\newblock In \emph{International conference on machine learning}, pages 551--560. PMLR, 2020.

\bibitem[Bourel et~al.(2023)Bourel, Jonsson, Maillard, and Talebi]{pmlr-v206-bourel23a}
Hippolyte Bourel, Anders Jonsson, Odalric-Ambrym Maillard, and Mohammad~Sadegh Talebi.
\newblock Exploration in reward machines with low regret.
\newblock In Francisco Ruiz, Jennifer Dy, and Jan-Willem van~de Meent, editors, \emph{Proceedings of The 26th International Conference on Artificial Intelligence and Statistics}, volume 206 of \emph{Proceedings of Machine Learning Research}, pages 4114--4146. PMLR, 25--27 Apr 2023.
\newblock URL \url{https://proceedings.mlr.press/v206/bourel23a.html}.

\bibitem[Bubeck and Cesa-Bianchi(2012)]{bubeck2012regret}
Sébastien Bubeck and Nicolò Cesa-Bianchi.
\newblock Regret analysis of stochastic and nonstochastic multi-armed bandit problems, 2012.

\bibitem[Camacho et~al.(2021)Camacho, Varley, Zeng, Jain, Iscen, and Kalashnikov]{camacho2021reward}
Alberto Camacho, Jacob Varley, Andy Zeng, Deepali Jain, Atil Iscen, and Dmitry Kalashnikov.
\newblock Reward machines for vision-based robotic manipulation.
\newblock In \emph{2021 IEEE International Conference on Robotics and Automation (ICRA)}, pages 14284--14290. IEEE, 2021.

\bibitem[Cesa-Bianchi and Lugosi(2006)]{cesa2006prediction}
Nicolo Cesa-Bianchi and G{\'a}bor Lugosi.
\newblock \emph{Prediction, learning, and games}.
\newblock Cambridge university press, 2006.

\bibitem[Dann and Brunskill(2015)]{dann2015sample}
Christoph Dann and Emma Brunskill.
\newblock Sample complexity of episodic fixed-horizon reinforcement learning.
\newblock \emph{Advances in Neural Information Processing Systems}, 28, 2015.

\bibitem[De~Santi et~al.(2024)De~Santi, Prajapat, and Krause]{de2024global}
Riccardo De~Santi, Manish Prajapat, and Andreas Krause.
\newblock Global reinforcement learning: Beyond linear and convex rewards via submodular semi-gradient methods.
\newblock \emph{arXiv preprint arXiv:2407.09905}, 2024.

\bibitem[DeFazio et~al.(2024)DeFazio, Hayamizu, and Zhang]{defazio2024learning}
David DeFazio, Yohei Hayamizu, and Shiqi Zhang.
\newblock Learning quadruped locomotion policies using logical rules.
\newblock In \emph{Proceedings of the International Conference on Automated Planning and Scheduling}, volume~34, pages 142--150, 2024.

\bibitem[Dohmen et~al.(2022)Dohmen, Topper, Atia, Beckus, Trivedi, and Velasquez]{dohmen2022inferring}
Taylor Dohmen, Noah Topper, George Atia, Andre Beckus, Ashutosh Trivedi, and Alvaro Velasquez.
\newblock Inferring probabilistic reward machines from non-markovian reward signals for reinforcement learning.
\newblock In \emph{Proceedings of the International Conference on Automated Planning and Scheduling}, volume~32, pages 574--582, 2022.

\bibitem[Freedman(1975)]{freedman1975martingale}
David~A. Freedman.
\newblock {On Tail Probabilities for Martingales}.
\newblock \emph{The Annals of Probability}, 3\penalty0 (1):\penalty0 100 -- 118, 1975.
\newblock \doi{10.1214/aop/1176996452}.
\newblock URL \url{https://doi.org/10.1214/aop/1176996452}.

\bibitem[Gheshlaghi~Azar et~al.(2013)Gheshlaghi~Azar, Munos, and Kappen]{gheshlaghi2013minimax}
Mohammad Gheshlaghi~Azar, R{\'e}mi Munos, and Hilbert~J Kappen.
\newblock Minimax pac bounds on the sample complexity of reinforcement learning with a generative model.
\newblock \emph{Machine learning}, 91:\penalty0 325--349, 2013.

\bibitem[Hu et~al.(2024)Hu, Xu, Wang, Qu, Pang, and Liu]{hu2024decentralized}
Jueming Hu, Zhe Xu, Weichang Wang, Guannan Qu, Yutian Pang, and Yongming Liu.
\newblock Decentralized graph-based multi-agent reinforcement learning using reward machines.
\newblock \emph{Neurocomputing}, 564:\penalty0 126974, 2024.

\bibitem[Icarte et~al.(2018)Icarte, Klassen, Valenzano, and McIlraith]{pmlr-v80-icarte18a}
Rodrigo~Toro Icarte, Toryn Klassen, Richard Valenzano, and Sheila McIlraith.
\newblock Using reward machines for high-level task specification and decomposition in reinforcement learning.
\newblock In Jennifer Dy and Andreas Krause, editors, \emph{Proceedings of the 35th International Conference on Machine Learning}, volume~80 of \emph{Proceedings of Machine Learning Research}, pages 2107--2116. PMLR, 10--15 Jul 2018.
\newblock URL \url{https://proceedings.mlr.press/v80/icarte18a.html}.

\bibitem[Icarte et~al.(2022)Icarte, Klassen, Valenzano, and McIlraith]{icarte2022reward}
Rodrigo~Toro Icarte, Toryn~Q Klassen, Richard Valenzano, and Sheila~A McIlraith.
\newblock Reward machines: Exploiting reward function structure in reinforcement learning.
\newblock \emph{Journal of Artificial Intelligence Research}, 73:\penalty0 173--208, 2022.

\bibitem[Jin et~al.(2020)Jin, Krishnamurthy, Simchowitz, and Yu]{jin2020rewardfree}
Chi Jin, Akshay Krishnamurthy, Max Simchowitz, and Tiancheng Yu.
\newblock Reward-free exploration for reinforcement learning, 2020.

\bibitem[Li et~al.(2017)Li, Vasile, and Belta]{li2017reinforcement}
Xiao Li, Cristian-Ioan Vasile, and Calin Belta.
\newblock Reinforcement learning with temporal logic rewards.
\newblock In \emph{2017 IEEE/RSJ International Conference on Intelligent Robots and Systems (IROS)}, pages 3834--3839. IEEE, 2017.

\bibitem[Liu et~al.(2021)Liu, Yu, Bai, and Jin]{liu2021sharp}
Qinghua Liu, Tiancheng Yu, Yu~Bai, and Chi Jin.
\newblock A sharp analysis of model-based reinforcement learning with self-play.
\newblock In \emph{International Conference on Machine Learning}, pages 7001--7010. PMLR, 2021.

\bibitem[Maurer and Pontil(2009)]{maurer2009empirical}
Andreas Maurer and Massimiliano Pontil.
\newblock Empirical bernstein bounds and sample variance penalization, 2009.

\bibitem[M{\'e}nard et~al.(2021)M{\'e}nard, Domingues, Jonsson, Kaufmann, Leurent, and Valko]{menard2021fast}
Pierre M{\'e}nard, Omar~Darwiche Domingues, Anders Jonsson, Emilie Kaufmann, Edouard Leurent, and Michal Valko.
\newblock Fast active learning for pure exploration in reinforcement learning.
\newblock In \emph{International Conference on Machine Learning}, pages 7599--7608. PMLR, 2021.

\bibitem[Neary et~al.(2020)Neary, Xu, Wu, and Topcu]{neary2020reward}
Cyrus Neary, Zhe Xu, Bo~Wu, and Ufuk Topcu.
\newblock Reward machines for cooperative multi-agent reinforcement learning.
\newblock \emph{arXiv preprint arXiv:2007.01962}, 2020.

\bibitem[Prajapat et~al.(2023)Prajapat, Mutn{\`y}, Zeilinger, and Krause]{prajapat2023submodular}
Manish Prajapat, Mojm{\'\i}r Mutn{\`y}, Melanie~N Zeilinger, and Andreas Krause.
\newblock Submodular reinforcement learning.
\newblock \emph{arXiv preprint arXiv:2307.13372}, 2023.

\bibitem[Qiu et~al.(2021)Qiu, Ye, Wang, and Yang]{qiu2021reward}
Shuang Qiu, Jieping Ye, Zhaoran Wang, and Zhuoran Yang.
\newblock On reward-free rl with kernel and neural function approximations: Single-agent mdp and markov game.
\newblock In \emph{International Conference on Machine Learning}, pages 8737--8747. PMLR, 2021.

\bibitem[Sadigh et~al.(2014)Sadigh, Kim, Coogan, Sastry, and Seshia]{sadigh2014learning}
Dorsa Sadigh, Eric~S Kim, Samuel Coogan, S~Shankar Sastry, and Sanjit~A Seshia.
\newblock A learning based approach to control synthesis of markov decision processes for linear temporal logic specifications.
\newblock In \emph{53rd IEEE Conference on Decision and Control}, pages 1091--1096. IEEE, 2014.

\bibitem[Sutton and Barto(2018)]{sutton2018reinforcement}
Richard~S Sutton and Andrew~G Barto.
\newblock \emph{Reinforcement learning: An introduction}.
\newblock MIT press, 2018.

\bibitem[Wagenmaker et~al.(2022)Wagenmaker, Chen, Simchowitz, Du, and Jamieson]{wagenmaker2022reward}
Andrew~J Wagenmaker, Yifang Chen, Max Simchowitz, Simon Du, and Kevin Jamieson.
\newblock Reward-free rl is no harder than reward-aware rl in linear markov decision processes.
\newblock In \emph{International Conference on Machine Learning}, pages 22430--22456. PMLR, 2022.

\bibitem[Wang et~al.(2020)Wang, Du, Yang, and Salakhutdinov]{wang2020reward}
Ruosong Wang, Simon~S Du, Lin Yang, and Russ~R Salakhutdinov.
\newblock On reward-free reinforcement learning with linear function approximation.
\newblock \emph{Advances in neural information processing systems}, 33:\penalty0 17816--17826, 2020.

\bibitem[Wen et~al.(2015)Wen, Ehlers, and Topcu]{wen2015correct}
Min Wen, R{\"u}diger Ehlers, and Ufuk Topcu.
\newblock Correct-by-synthesis reinforcement learning with temporal logic constraints.
\newblock In \emph{2015 IEEE/RSJ International Conference on Intelligent Robots and Systems (IROS)}, pages 4983--4990. IEEE, 2015.

\bibitem[Xu et~al.(2022)Xu, Gavran, Ahmad, Majumdar, Neider, Topcu, and Wu]{xu2022joint}
Zhe Xu, Ivan Gavran, Yousef Ahmad, Rupak Majumdar, Daniel Neider, Ufuk Topcu, and Bo~Wu.
\newblock Joint inference of reward machines and policies for reinforcement learning, 2022.

\bibitem[Zanette and Brunskill(2019)]{zanette2019tighter}
Andrea Zanette and Emma Brunskill.
\newblock Tighter problem-dependent regret bounds in reinforcement learning without domain knowledge using value function bounds.
\newblock In \emph{International Conference on Machine Learning}, pages 7304--7312. PMLR, 2019.

\bibitem[Zhang et~al.(2021{\natexlab{a}})Zhang, Zhou, and Gu]{zhang2021reward}
Weitong Zhang, Dongruo Zhou, and Quanquan Gu.
\newblock Reward-free model-based reinforcement learning with linear function approximation.
\newblock \emph{Advances in Neural Information Processing Systems}, 34:\penalty0 1582--1593, 2021{\natexlab{a}}.

\bibitem[Zhang et~al.(2020)Zhang, Ma, and Singla]{zhang2020task}
Xuezhou Zhang, Yuzhe Ma, and Adish Singla.
\newblock Task-agnostic exploration in reinforcement learning.
\newblock \emph{Advances in Neural Information Processing Systems}, 33:\penalty0 11734--11743, 2020.

\bibitem[Zhang et~al.(2021{\natexlab{b}})Zhang, Ji, and Du]{zhang2021reinforcement}
Zihan Zhang, Xiangyang Ji, and Simon Du.
\newblock Is reinforcement learning more difficult than bandits? a near-optimal algorithm escaping the curse of horizon.
\newblock In \emph{Conference on Learning Theory}, pages 4528--4531. PMLR, 2021{\natexlab{b}}.

\bibitem[Zheng and Yu(2024)]{zheng2024multi}
Xuejing Zheng and Chao Yu.
\newblock Multi-agent reinforcement learning with a hierarchy of reward machines.
\newblock \emph{arXiv preprint arXiv:2403.07005}, 2024.

\end{thebibliography}
\onecolumn

\appendix

\section{Table of notation}\label{appendix:ton}
\begin{table*}[h!] 
\centering 
\label{tab:notation}
\begin{tabular}{cl} 
\toprule
\textbf{Symbol} & \textbf{Explanation} \\
\midrule
$\calO$ & The observation space  \\
$\calQ$ & The state space of PRM \\
$\calS$ & $\calQ\times\calO$\\
$\calA$ & The action space  \\
$\pi_k$ & The policy at episode k  \\
$P$     & The transition function on $\calS$\\
$R$     & The reward function\\
$p$     & The transition function on $\calO$\\
$\tau$  & The transition function of PRM\\
$\nu$   & The reward function of PRM\\
$L$     & Labeling function\\
$O$     & Size of observation space\\
$A$     & Size of action space\\
$H$     & The horizon length\\
$Q$     & Size of state space of PRM\\
$T$ and $T_k$ & The total number of steps and number of steps up to episode $k$\\
$K$     & The total number of episodes\\
$N_k(o,a)$ & Number of visits to observation-action pair $(o,a)$ up to episode $k$\\
$V_h^*$ & Optimal value function $V^*$\\
$V_{k,h}$ & The estimate of value function at step $h$ of episode $k$\\
$W_h^*(q, o, a, o')$ & $\sum_{q'} \tau(q'|q, L(o, a, o'))V^*_{h}(q',o')$ \\
$W_{k,h}(q, o, a, o')$ & $\sum_{q'} \tau(q'|q, L(o, a, o'))V_{k,h}(q',o')$\\
$Q_{k,h}$ & The estimate of action function at step $h$ of episode $k$\\
$b$     & The exploration bonus\\
$\iota$ & $\ln{\left(\frac{6QOAT}{\rho}\right)}$\\
$N_k(o, a, o')$ & Number of transitions from $o$ to $o'$ after taking action $a$ up to episode $k$\\
$N'_{k, h}(o, a, o')$ & Number of transitions from $o$ to $o'$ after taking action $a$ at step $h$ up to episode $k$\\
$N'_{k, h}(o)$ & Number of visits to observation $o$ at step $h$ up to episode $k$\\
$\widehat{p}_{k}(o'|o, a)$ & The empirical transition distribution from $o$ to $o'$ upon taking action $a$ of episode $k$\\
$\widehat{\mathbbV}_{k, h}(s, a)$ &  The empirical next-state variance of $V_{k, h}$ for every $(s, a)$\\
$\mathbbV_h^*(s, a)$ & The next-state variance of $V^*$ for every state-action pair $(s, a)$\\
$\widehat{\mathbbV}_{k, h}^*(s, a)$ & The empirical next-state variance of $V_{h}^*$ for every $(s, a)$ at episode $k$\\
$\mathbbV_h^{\pi}(s, a)$ & The next-state variance of $V_h^{\pi}$ for every state-action pair $(s, a)$\\
$\widehat{\mathbbW}_{k, h}(q, o, a)$ & The empirical next-observation variance of $W_{k, h}$ for every triplet $(q, o, a)$\\
$\mathbbW_{h}^*(q, o, a)$ & The next-observation variance of $W_h^*$ for every triplet $(q, o, a)$\\
$\widehat{\mathbbW}^*_{k, h}(q, o, a)$ & The empirical next-observation variance of $W_h^*$ for every triplet $(q, o, a)$ at episode $k$\\
$\mathbbW_{h}^{\pi}(q, o, a)$ & The next-observation variance of $W_h^{\pi}$ for every triplet $(q, o, a)$\\
\Regret$(k)$ & The regret after $k$ episodes\\
$\widetilde{\Regret}(k)$ & The upper-bound regret after $k$ episodes\\
$\Delta(k, h)$ & One step regret at step $h$ of episode $k$\\
$\widetilde{\Delta}(k, h)$ & One step upper-bound regret at step $h$ of episode $k$ \\
$\calE$ & The high probability event under which the concentration inequalities hold\\
$\Omega$ & The high probability event under which the $V_{k, h}$ are the upper bounds of optimal value function\\
$\calH_t$ & The history of all random events up to time step $t$\\
\bottomrule
\end{tabular}
\end{table*}
\section{Notation}\label{appendix:n}
We follow the notations of \citet{azar2017minimax}.
\begin{align*}
    W_{k, h}(q, o, a, o') &\defeq \sum_{q'} \tau(q'|q, L(o, a, o'))V_{k, h}(q',o')\\
    W^*_{h}(q, o, a, o') &\defeq \sum_{q'} \tau(q'|q, L(o, a, o'))V^*_{h}(q',o')\\
    \phi(q, o, a, o') &\defeq \sum_{q'} \tau(q'|q, L(o, a, o'))\nu(q, L(o,a,o'),q')\\
    \widehat{\mathbbW}_{k, h}(q, o, a) &\defeq  \operatorname{Var}_{o' \sim \widehat{p}_k(\cdot|o,a)} (W_{k, h}(q, o, a, o'))\\
    \mathbbW_{h}^*(q, o, a) &\defeq  \operatorname{Var}_{o' \sim p(\cdot|o,a)} (W_{h}^*(q, o, a, o'))\\
    \widehat{\mathbbW}^*_{k, h}(q, o, a) &\defeq  \operatorname{Var}_{o' \sim \widehat{p}_k(\cdot|o,a)} (W^*_{h}(q, o, a, o'))\\
    \mathbbW_{h}^{\pi}(q, o, a) &\defeq  \operatorname{Var}_{o' \sim p(\cdot|o,a)} (W_{h}^{\pi}(q, o, a, o'))\\
    \widehat{\mathbbV}_{k, h}(s, a) &\defeq  \operatorname{Var}_{y \sim \widehat{P}_k(\cdot|s,a)} (V_{k,h+1}(y))\\
    \mathbb{V}_h^*(s, a) &\defeq  \operatorname{Var}_{y \sim P(\cdot|s,a)} (V_{h}^*(y))\\
    \widehat{\mathbbV}^*_{k, h}(s, a) &\defeq \operatorname{Var}_{y \sim \widehat{P}(\cdot|s,a)} (V_{h}^*(y))\\
    \mathbb{V}_h^{\pi}(s, a) &\defeq  \operatorname{Var}_{y \sim P(\cdot|s,a)} (V_{h}^{\pi}(y))
\end{align*}
 We use the lower case to denote the functions evaluated at the current state-action pair. \eg let $n_{k, h} = N_{k}(o_{k, h}, \pi_k(s_{k, h}, h))$. We also define $b'_{i, j}(o) = \min{\left(\frac{100^2O^2H^2A\iota^2}{N'_{i, j+1}(o)}, H^2\right)}$. We also denote $pW(q,o,a) = \sum_{o'\in \calO}p(o'|o,a)W(q, o, a, o')$ and $p\phi(q,o,a) = \sum_{o'\in \calO}p(o'|o,a)\phi(q, o, a, o')$. 
\subsection{"Typical" state-actions and steps}
We split the episodes into $2$ sets: the set of "typical" episodes in which the number of visits to the encountered state-actions are large and the rest of the episodes. Further, we define
\begin{align*}
    [(o, a)]_k \defeq \{(o,a): (o,a)\in \calO\times\calA, N_h(o, a)\geq H, N'_{k,h}(o) \geq H\}
\end{align*}
We define typical episodes and the set of typical state-dependent episodes as follows
\begin{align*}
    [k]_{\text{typ}} &\defeq \left\{i: i\in [k], \forall h\in[H], (o_{i, h}, \pi_i(q_{i, h}, o_{i, h}, h))\in[(o, a)]_k, i\geq 650HO^2A^2\iota\right\}\\
    [k]_{\text{typ}, o} &\defeq \left\{i: i\in [k], \forall h\in[H], (o_{i, h}, \pi_i(q_{i, h}, o_{i, h}, h))\in[(o, a)]_k, N'_{k, h}(o)\geq 650HO^2A^2\iota\right\}
\end{align*}
\subsection{surrogate regrets}
Define $\Delta_{k, h} \defeq V_h^* - V_h^{\pi_k}$, $\widetilde{\Delta}_{k,h} \defeq V_{k,h} - V_h^{\pi_k}$, $\delta_{k, h} \defeq V_h^*(s_{k, h}) - V_h^{\pi_k}(s_{k, h})$, and $\widetilde{\delta}_{k, h} \defeq V_{k, h}(s_{k, h}) - V_h^{\pi_k}(s_{k, h})$ for every $s \in \calS$, $h \in [H]$ and $k \in [K]$. We denote Then the upper-bound regret $\widetilde{\Regret}(k)$ is defined as follows
\begin{align*}
    \widetilde{\Regret}(k) \defeq \sum_{i=1}^k \widetilde{\delta}_{i, 1}
\end{align*}
We also define regret of every state and its corresponding upper bounds as
\begin{align*}
    \Regret(k, s, h) &\defeq \sum_{i = 1}^k \indicator(s_{i, h} = s) \delta_{i, h}\\
    \widetilde{\Regret}(k, s, h) &\defeq \sum_{i = 1}^k \indicator(s_{i, h} = s) \widetilde{\delta}_{i, h}\\
\end{align*}
\subsection{Martingale difference sequences}
We define the following martingale operator for every $k \in [K]$, $h \in [H]$ and $F:\calS \rightarrow \mathbb{R}$, let $t = (k-1)H + h$ denote the time stamp at step $h$ of episode $k$ then
\begin{align*}
    \calM_t F\defeq P_h^{\pi_k}F - F(s_{k, h+1}).
\end{align*}
Let $\calH_t$ be the history of all random events up to (and including) step $h$ of episode $k$ then we have that $\mathbbE(\calM_t F|\calH_t) = 0$. Hence $\calM_t F$ is a martingale difference w.r.t. $\calH_t$.
\subsection{High probability events}
Denote $\iota = \ln{\left(\frac{6QOAT}{\rho}\right)}$. We define the high probability events $\calE$ and $\Omega_{k, h}$. We define four confidence levels as follows.
\begin{align*}
    c_1(v, u, n) &\defeq 2 \sqrt{\frac{v\iota}{n}} + \frac{14u\iota}{3n}\\
    c_2(p, n) &\defeq \sqrt{\frac{2p(1-p)\iota}{n}} + \frac{2\iota}{3n}\\
    c_3(n) &\defeq 2\sqrt{\frac{O\iota}{n}}\\
    c_4(u, n) &\defeq u\sqrt{\frac{\iota}{n}}
\end{align*}
Let $\calP$ be the set of all probability distributions on $\calO$. Define the following confidence set for every $k = 1, 2, \ldots, K,$ $n > 0$ and $(o, a) \in \calO \times \calA$. 
\begin{align*}
    \calP(k, h, n, o, a, z) \defeq &\{\Tilde{p}(\cdot|s, a) \in \calP: \forall q\in \calQ\\
    &\left|(\Tilde{p}-p)W_h^*(q, o, a)\right| \leq \min{\left(c_1(\mathbbW_h^*(s, a), H, n), c_1(\widehat{\mathbbW}_{k, h}^*(s, a), H, n)\right)}\\
    &\left|\Tilde{p}(z|o, a) - p(z|o, a)\right| \leq c_2(p(z|o, a), n)\\
    &\left\|\Tilde{p}(\cdot|o, a) - p(\cdot|o, a)\right\|_1\leq c_3(n)\\
    &\left|(\Tilde{p}-p)\phi(q, o, a)\right| \leq c_4(1, n)\}
\end{align*}
We now define the random event $\calE_{\widehat{p}}$ as follows
\begin{align*}
    \calE_{\widehat{p}} \defeq \{\widehat{p}_k(z|o,a)\in \calP(k, h, N_k(o, a), o, a, z), \forall k\in[K],\forall h\in [H], \forall(z,o,a)\in \calO\times\calO\times\calA\}
\end{align*}
Let $t$ be a positive integer, and let \(\calF = \{f_s\}_{s \in [t]}\) be a set of real-valued functions defined on \(\calH_{t+s}\) for some integer  $s>0$ . We define the following random events for given parameters \(\Bar{w} > 0\), \(\Bar{u} > 0\), and \(\Bar{c} > 0\):
\begin{align*}
    \calE_{\text{az}}(\calF, \Bar{u}, \bar{c}) &\defeq \left\{\sum_{s=1}^t \calM_s f_s \leq 2\sqrt{t\bar{u}^2 \Bar{c}}\right\}\\
    \calE_{\text{fr}}(\calF, \Bar{w}, \Bar{u}, \bar{c}) &\defeq \left\{\sum_{s=1}^t \calM_s f_s \leq 4\sqrt{\Bar{w}c} + \frac{14\Bar{u}\Bar{c}}{3}\right\}
\end{align*}
We use the short-hand notation $\calE_{\text{az}}(\calF, \Bar{u})$ and $\calE_{\text{fr}}(\calF, \Bar{w}, \Bar{u})$ for $\calE_{\text{az}}(\calF, \Bar{u}, \iota)$ and $\calE_{\text{fr}}(\calF, \Bar{w}, \Bar{u}, \iota)$. We define a set of random variables for every $k \in [K]$ and $h\in[H]:$
\begin{align*}
    \calF_{\widetilde{\Delta}, k, h } &\defeq \left\{\widetilde{\Delta}_{i, j}: i\in [K], h<j\in[H]\right\}\\
    \calF_{\widetilde{\Delta}, k, h, s} &\defeq \left\{\widetilde{\Delta}_{i, j}\indicator(s_{i, j} = s): i\in [K], h<j\in[H]\right\}\\
    \calG_{\mathbbW, k, h } &\defeq \left\{\sum_{j=h+1}^H \mathbbW_j^{\pi_i}: i\in [K], h<j\in[H]\right\}\\
    \calG_{\mathbbW, k, h, s} &\defeq \left\{\sum_{j=h+1}^H \mathbbW_j^{\pi_i}\indicator(s_{i, j} = s): i\in [K], h<j\in[H]\right\}\\    
    \calF_{b', k, h} &\defeq \left\{b'_{i, j}: i\in [K], h<j\in[H]\right\}\\
    \calF_{b', k, h, s} &\defeq \left\{b'_{i, j}\indicator(s_{i, j} = s): i\in [K], h<j\in[H]\right\}
\end{align*}
We now define the event $\calE$ as follows
\begin{align*}
\calE \defeq \calE_{\widehat{P}} \bigcap \bigcap_{\substack{k \in [K],\\ h \in [H], \\s \in \calS}} &\left[\calE_{\text{az}}(\calF_{\widetilde{\Delta},k,h}, H) \bigcap \calE_{\text{az}}(\calF_{\widetilde{\Delta},k, h, s}, H) \bigcap \calE_{\text{fr}}(\calG_{\mathbbW,k,h}, H^4 T, H^3) \bigcap \calE_{\text{fr}}(\calG_{\mathbbW,k,h,s}, H^5N'_{k,h}(s), H^3) \right.\\
&\left. \bigcap \calE_{\text{az}}(\calF_{b',k,h}, H^2) \bigcap \calE_{\text{az}}(\calF_{b',k,h,s}, H^2)\right]
\end{align*}
where $N'_{k, h}(s)$ is number of visits to observation $s$ at step $h$ up to episode $k$, correspondingly, we have $N'_{k,h}(o)$ as number of visits to observation $o$ at step $h$ up to episode $k$.
\subsubsection{UCB Events}
Let $k \in [K]$ and $h \in [H]$. Denote the set of steps for which the value functions are obtained before $V_{k, h}$ as
\begin{align*}
    [k, h]_{hist} = \{(i, j): i\in [K], j \in [H], i\leq k \vee(i=k\wedge j>h)\}
\end{align*}
Let $\Omega_{k, h}=\{V_{i,j} \geq V_h^*, \forall (i, j)\in [k, h]_{hist}\}$ be the event under which $V_{i, j}$ are the upper bound of the optimal value functions. We will prove this later.
\begin{lemma}\label{lemma_var}
Let $s = (q, o), s' = (q', o')$, then the following inequality holds
    \begin{align*}
        \Var_{o'\sim p(\cdot|o, a)} W(q,o,a,o') \leq \Var_{s'\sim P(\cdot|s, a)} V(s')
    \end{align*}
\end{lemma}
\begin{proof}
    By definition of $W$ and $V$, it is easy to verify that
    \begin{align*}
        \mathbbE(W(q, o, a, o')) = \sum_{o'}p(o'|o, a) W(q, o, a, o') = \sum_{s'} P(s'|s, a) V(s') = \mathbbE(V(s'))
    \end{align*}
    Then we can have
    \begin{align*}
        &\Var_{s'\sim P(\cdot|s, a)} V(s') = \mathbbE(V^2(s')) - \mathbbE(V(s'))^2\\
        &= \sum_{o'}p(o'|o, a) \sum_{q'} \tau(q'|q, L(o, a, o'))V^2(q',o')  - \mathbbE(V(s'))^2\\
        &\geq \sum_{o'}p(o'|o, a) \left(\sum_{q'} \tau(q'|q, L(o, a, o'))V(q',o')\right)^2 - \mathbbE(V(s'))^2\\
        &= \mathbbE(W^2(q, o, a, o')) - \mathbbE(W(q, o, a, o'))^2 = \Var_{o'\sim p(\cdot|o, a)} W(q,o,a,o')
    \end{align*}
    The inequality holds because of Jensen's inequality.
\end{proof}
\begin{lemma}\label{lem:calE}
    Let $\rho > 0$ be a real scalar. Then the event $\calE$ holds w.p. at least $1-\rho$.
\end{lemma}
\begin{proof}
    We need a set of concentration inequalities to prove this event holds with high probability. For a fixed $(k, h, q, o, a)$ tuple, following inequality holds with probability at least $1-\rho$ per Bernstein inequality(see, \eg \citet{cesa2006prediction}, \citet{bubeck2012regret}). Denote $s = (q, o) \in \calS$.
    \begin{align*}
        \left|(\Tilde{P}-P)V_h^*(s, a)\right| &= \sum_{o'}\left|(\Tilde{p}(o'|o, a)-p(o'|o, a))\right|\sum_{q'} \tau(q'|q, L(o, a, o'))V^*_{h}(q',o')\\
        &= \sum_{o'}\left|(\Tilde{p}(o'|o, a)-p(o'|o, a))\right| W_h^*(q,o,a,o') = \left|(\Tilde{p}-p)W_h^*(q, o, a)\right| \\
        &\leq \sqrt{\frac{2\mathbbW_h^*(q, o, a)\ln\frac{2}{\rho}}{N_k(o, a)}}+ \frac{2H\ln{\frac{2}{\rho}}}{3N_k(o, a)}
    \end{align*}
    where $W_h^* \leq H$.
    Using empirical Bernstein's inequality\citep{maurer2009empirical}, we can have following result by similar argument for $N_k(o, a) > 1$,
    \begin{align*}
        \left|(\Tilde{P}-P)V_h^*(s, a)\right| &\leq \sqrt{\frac{2\widehat{\mathbbW}_{k,h}^*(q, o, a)\ln\frac{2}{\rho}}{N_k(o, a) - 1}}+ \frac{7H\ln{\frac{2}{\rho}}}{3(N_k(o, a)-1)}
    \end{align*}
    The Bernstein inequality also implies with probability at least $1-\rho$
    \begin{align*}
        \left|\widehat{p}_k(z|o, a) - p(z|o, a)\right| \leq \sqrt{\frac{{2p(z|o, a)(1-p(z|o, a))}\ln\frac{2}{\rho}}{N_k(o, a)}} + \frac{2\ln{\frac{2}{\rho}}}{3N_k(o, a)}
    \end{align*}
    where $ \operatorname{Var}_{o' \sim P(\cdot|o,a)} \indicator(z=o') = P(z|o, a)(1-P(z|o, a))$.\\
    Using Azuma-Hoeffding's inequality(\eg \citet{cesa2006prediction}), we can have the following for a fixed $(k, h, q, o, a)$
    \begin{align*}
        \left| (\Tilde{P} - P)r(s, a)) \right| =  \left|(\Tilde{p}-p)\phi(q, o, a)\right| \leq  \sqrt{\frac{\ln(\frac{2}{\rho})}{N_k(o, a)}}.
    \end{align*}
    Let the sequence of random variables $\{X_1, X_2, \ldots, X_n\}$ be a martingale difference sequence w.r.t. some filtration $\calF_n$. When $\left|X_i\right| \leq u, \forall i \in [n]$, by applying the Azuma's inequality\citep{cesa2006prediction}, we will have with probability at least $1-\rho$
    \begin{align*}
        \sum_{i=1}^n X_i \leq \sqrt{2nu\ln\frac{1}{\rho}}
    \end{align*}
    Then this inequality implies that the following events hold with probability at least $1-\rho$
    \begin{align*}
        &\calE_{\text{az}}(\calF_{\widetilde{\Delta},k,h}, H, \ln{(1/\rho)})\\
        &\calE_{\text{az}}(\calF_{b',k,h}, H^2, \ln{(1/\rho)})
    \end{align*}
    Combined with a union bound over all $N'_{k, h}(s)\in[T]$, the above events further implies the following events hold with probability at least $1-\rho$
    \begin{align*}
        &\calE_{\text{az}}(\calF_{\widetilde{\Delta},k,h, s}, H, \ln{(T/\rho)})\\
        &\calE_{\text{az}}(\calF_{b',k,h, s}, H^2, \ln{(T/\rho)})
    \end{align*}
    When the sum of the variances $\sum_{i=1}^n \operatorname{Var}(X_i|\calF_i)\leq w$ for some $w > 0$, then the following holds by Freedman\citep{freedman1975martingale} w.p. $1-\rho$
    \begin{align*}
        \sum_{i=1}^n X_i \leq \sqrt{2w\ln\frac{1}{\rho}} + \frac{2u\ln\frac{1}{\rho}}{3}
    \end{align*}
    We now prove that the following events hold w. p. $1-\rho$
    \begin{align*}
        &\calE_{\text{fr}}(\calG_{\mathbbW, k, h}, \Bar{w}_{k, h}, H^3, \ln{(1/\rho)})\\
        &\calE_{\text{fr}}(\calG_{\mathbbW, k, h, s}, \Bar{w}_{k, h, s}, H^3, \ln{(T/\rho)})
    \end{align*}
    where $\Bar{w}_{k, h}$ and $\Bar{w}_{k, h, s}$ are the upper bound on $Z_{k, h}$ and $Z_{k, h, s}$ respectively, which are defined as
    \begin{align*}
        &Z_{k, h} \defeq \sum_{i=1}^{k} \Var\left( \sum_{j=h}^{H-1} \mathbbW^{\pi}_{i, j+1}\middle\vert\calH_{i, 1}\right)\\
        &Z_{k, h, s} \defeq \sum_{i=1}^{k} \indicator(s_{i,h} = s)\Var\left( \sum_{j=h}^{H-1} \mathbbW^{\pi}_{i, j+1}\middle\vert\calH_{i, 1}\right)\\
    \end{align*}
    Since $W^{\pi}_{k, h}(q, o, a, o') = \sum_{q'} \tau(q'|q, L(o, a, o'))V^{\pi}_{k, h}(q',o') \leq H$
    \begin{align*}
        Z_{k, h} \leq \sum_{i=1}^k \mathbbE \left(\sum_{j=h}^{H-1} \mathbbW^{\pi}_{i, j+1}\middle\vert\calH_i\right)^2 \leq H^3  \sum_{i=1}^k \mathbbE \left(\sum_{j=h}^{H-1} \mathbbW^{\pi}_{i, j+1}\middle\vert\calH_i\right)
    \end{align*}
    Using lemma \ref{lemma_var}, we have
    \begin{align*}
        \mathbbE \left(\sum_{j=h}^{H-1} \mathbbW^{\pi}_{i, j+1}\middle\vert\calH_i\right) \leq \mathbbE \left(\sum_{j=h}^{H-1} \mathbbV^{\pi}_{i, j+1}\middle\vert\calH_i\right)
    \end{align*}
    Let the sequence $\{s_1, s_2, \ldots, s_H\}$ be the sequence of states by following policy $\pi$ on episode $k$. Then applying recursive application of Law of Total Variance(see \eg lemma 7 of \citet{gheshlaghi2013minimax}), leads to
    \begin{align*}
        \mathbbE \left(\sum_{j=h}^{H-1} \mathbbV^{\pi}_{i, j+1}(s_j)\right) = \Var\left(\sum_{j=h}^{H-1} r^{\pi}(s_j)\right)
    \end{align*}
    Hence we can deduce that
    \begin{align*}
        Z_{k, h} \leq H^3\sum_{i=1}^k\Var\left(\sum_{j=h}^{H-1} r_{i, j}\middle\vert\calH_i\right)\leq H^5 k = H^4 T_k
    \end{align*}
    Similarly,
    \begin{align*}
        Z_{k, h, s} \leq H^5N_{k,h}(s)
    \end{align*}
    Hence, the following event hold w.p. $1-\rho$
    \begin{align*}
        &\calE_{\text{fr}}(\calG_{\mathbbW, k, h}, H^4T, H^3, \ln{(1/\rho)})\\
        &\calE_{\text{fr}}(\calG_{\mathbbW, k, h, s}, H^5N_{k,h}(s), H^3, \ln{(T/\rho)})
    \end{align*}
    Combining all the results above and take a union bound argument over all $k\in[K]$, $h\in[H]$, $a\in\calA$, $q\in\calQ$ and $o\in\calO$, we prove that events $\calE$ holds w.p. $1-\rho$.
\end{proof}
\subsection{Other useful notation}
We denote the total count of steps up to episode $k\in[K]$ by $T_k\defeq H(k-1)$. We define $\xi_{k, h}$ for every $h\in[H]$ and $k\in [K]$, as follows
\begin{align*}
    \xi_{k, h} = \frac{5H^2O\iota}{3N_k(o_{k, h}, \pi_k(s_{k, h}, h))}
\end{align*}
For every $k\in [K]$, $h\in[H]$ and $s\in\calS$, we introduce the following notation.
\begin{align*}
    C_{k,h} &\defeq \sum_{i=1}^{k}\indicator(i\in[k]_{\text{typ}})\sum_{j=h}^{H-1} e_{i, j}\\
    C_{k, h, s} &\defeq \sum_{i=1}^{k}\indicator(i\in[k]_{\text{typ}},s_{i, h} = s)\sum_{j=h}^{H-1} e_{i, j}\\
    B_{k,h} &\defeq \sum_{i=1}^{k}\indicator(i\in[k]_{\text{typ}})\sum_{j=h}^{H-1} b_{i, j}\\
    B_{k,h, s} &\defeq \sum_{i=1}^{k}\indicator(i\in[k]_{\text{typ}},s_{i, h} = s)\sum_{j=h}^{H-1} b_{i, j}\\
    C'_{k, h} &\defeq \sum_{i=1}^{k}\sum_{j=h}^{H-1} \Bar{e}_{i, j}
\end{align*}
We also define the upper bound $U_k,h$ for every $k \in [K]$, $h \in [H]$ and $s\in \calS$ as follows
\begin{align*}
    U_{k, h} \defeq e\sum_{i=1}^k \sum_{j=h}^{H-1}\left[b_{i, j}+e_{i, j} + \Bar{e}_{i, j} + \xi_{i,j}\right] + H\sqrt{T_k\iota}
\end{align*}
\section{Proof of the Regret Bounds}\label{appendix:rg}
\begin{lemma}\label{lemma_err_model}
    Fix \(\rho \in (0, 1)\). Let the event $\calE$ hold. For all \(k \in [0, \ldots, K-1]\), \(o \in \mathcal{O}\), \(q \in \mathcal{Q}\), \(s = (q, o) \in \mathcal{S}\), \(a \in \mathcal{A}\), \(h \in [0, \ldots, H-1]\), we have that for any \(f : \mathcal{S} \rightarrow [0, H]\):
\begin{align*}
    \left| (\widehat{P}_{k} - P)f(s, a) \right| \leq \frac{\sum_{s'}P(s'|s,a)f(s')}{H} + \frac{5H^2O\iota }{3N_{k}(o, a)}.    
\end{align*}
\end{lemma}
\begin{proof}
    \begin{align*}
        \left| (\widehat{P}_{k}(s,a) - P(s,a))^T f \right|
        &\leq \sum_{o'}\left|\widehat{p}_{k}(o'|o, a) - p(o'|o, a)\right| \sum_{q'}\tau(q'|q, L(o, a, o'))f(q', o')\\
        &= \sum_{o'}\left|\widehat{p}_{k}(o'|o, a) - p(o'|o, a)\right| F(q,o,a,o')
    \end{align*}
    where we denote $F(q,o,a,o') = \sum_{q'}\tau(q'|q, L(o, a, o'))f(q', o') \leq H$, 
    \begin{align*}
        &\sum_{o'}\left|\widehat{p}_{k}(o'|o, a) - p(o'|o, a)\right| F(q,o,a,o')
        \leq \sum_{o'}\left( \sqrt{\frac{2\iota p(o'|o, a)}{N_{k}(o, a)}} + \frac{2\iota }{3N_{k}(o, a)}\right) F(q,o,a,o')\\
        &\leq \sum_{o'}\sqrt{\frac{2\iota p(o'|o, a)F^2(q,o,a,o')}{N_{k, h}(o, a)}} + \frac{2OH\iota }{3N_{k}(o, a)}
        \leq \sqrt{O}\sqrt{\frac{\sum_{o'}2\iota p(o'|o, a)F^2(q,o,a,o')}{N_{k}(o, a)}}+ \frac{2OH\iota }{3N_{k}(o, a)}\\
        &=\sqrt{\frac{2H^2O\iota }{N_{k}(o, a)}\cdot\frac{\sum_{o'}p(o'|o, a)F^2(q,o,a,o')}{H^2}}+\frac{2OH\iota }{3N_{k}(o, a)}\\
        &\leq \frac{H^2O\iota }{N_{k}(o, a)} + \frac{\sum_{o'}p(o'|o, a)F^2(q,o,a,o')}{H^2} + \frac{2OH\iota }{3N_{k}(o, a)}\\
        &\leq \frac{\sum_{o'}p(o'|o, a)F(q,o,a,o')}{H} + \frac{5H^2O\iota }{3N_{k}(o, a)}
        = \frac{\sum_{s'}P(s'|s,a)f(s')}{H} + \frac{5H^2O\iota }{3N_{k}(o, a)}
    \end{align*}
    The first inequality holds under the event $\calE$. The pigeonhole principle is applied to the second inequality, the Cauchy-Schwarz inequality to the third, and the fact that \(ab \leq \frac{a^2 + b^2}{2}\) to the fourth. The final inequality uses the condition \(F(q,o,a,o') \leq H\).
\end{proof}
\begin{lemma}\label{lemma_total_regret}
    Let $k \in [K]$ and $h\in [H]$. Let the events $\calE$ and $\Omega_{k, h}$ hold. Then the following inequalities hold for $\delta_{k, h}$ and $\widetilde{\delta}_{k, h}$:
    \begin{align*}
        \delta_{k, h} \leq \widetilde{\delta}_{k, h} \leq e \sum_{i=h}^{H-1}(\epsilon_{i} + \xi_{i} +  b_{i} + e_{i} + \Bar{e}_{i})
    \end{align*}
\end{lemma}
\begin{proof}
    This analysis builds on the regret analysis of \citet{azar2017minimax}, but introduces novel design for the property of PRMs.\\
    Define $\Delta_{k, h} \defeq V_h^* - V_h^{\pi_k}$, $\widetilde{\Delta}_{k,h} \defeq V_{k,h} - V_h^{\pi_k}$. We drop the dependencies of notation on episode $k$ for easier presentation, \eg we write $o_{k, h}$ and $V_{k, h}$ as $o_h$ and $V_h$.\\ 
    \begin{align*}
        \Delta_{h} &\leq \widetilde{\Delta}_{h} = \widehat{r}^{\pi}_{h} + \widehat{P}_{h}^{\pi} V_{  h+1} + b_{h} - r^{\pi}_{h} - P_{h}^{\pi}V_{h+1}^{\pi} \\
        &= (\widehat{P}_{h}^{\pi} - P_{h}^{\pi}) V_{h+1} + P_{h}^{\pi}\widetilde{\Delta}_{h+1} + b_{h} + \widehat{r}^{\pi}_{h} - r^{\pi}_h\\
        &= (\widehat{P}_{h}^{\pi} - P_{h}^{\pi})(V_{h+1} - V^*_{h+1}) + P_{h}^{\pi_k}\widetilde{\Delta}_{h+1} + b_{h} + e_{h} + \Bar{e}_{h}
    \end{align*}
    where for $s = (q, o)$, $\widehat{r}^{\pi}_{h}(s, a) = \sum_{o', q'} \widehat{p}_{h}(o'|o,\pi(s, h))\nu(q, L(o,\pi(s, h),o'), q')$, $r^{\pi}_h(s, a) =\sum_{o', q'} p_{h}(o'|o,\pi(s, h))\nu(q, L(o,\pi(s, h),o'), q')$, $e_{h} \defeq (\widehat{P}_{h}^{\pi} - P_{h}^{\pi})V_{h+1}^*$ is the estimation error of the optimal value function at the next state. $\Bar{e}_{h} \defeq \widehat{r}^{\pi}_{h} - r^{\pi}_{h}$ is the estimation error of the reward function. \\
    $\widetilde{\delta}_{h} = \widetilde{\Delta}_{h}(s_{h})$, we further have
    \begin{align*}
         \widetilde{\delta}_{ h} &= (\widehat{P}_{ h}^{\pi} - P_{h}^{\pi})(V_{ h+1} - V^*_{h+1}) (s_{ h}) + P_{h}^{\pi}\widetilde{\Delta}_{h+1} (s_{ h}) +  b_{h}(s_h) + e_{h}(s_h) + \Bar{e}_{h}(s_h)\\
         &= \underbrace{(\widehat{P}_{ h} - P_{h})(V_{ h+1} - V^*_{h+1}) (s_{ h}, \pi(s_{ h}))}_{(a)} + P_{h}^{\pi}\widetilde{\Delta}_{h+1} (s_{ h}) +  b_{h}(s_h) + e_{h}(s_h) + \Bar{e}_{h}(s_h)
    \end{align*}
    Based on lemma \ref{lemma_err_model},
    \begin{align*}
        (a) \leq \frac{P_h^{\pi}(V_{h+1} - V^*_{h+1}) (s_{h})}{H} + \xi_{h}
    \end{align*}
    where $\xi_{h} = \frac{5H^2O\iota }{3N_{k}(o, a)}$. Hence
    \begin{align*}
         &\widetilde{\delta}_{ h} \leq \frac{P_h^{\pi}(V_{h+1} - V^*_{h+1}) (s_{h})}{H} + \xi_{h} + P_{h}^{\pi}\widetilde{\Delta}_{h+1} (s_{h}) +  b_{h}(s_h) + e_{h}(s_h) + \Bar{e}_{h}(s_h)\\
         &\leq (1+\frac{1}{H})P_h^{\pi}\widetilde{\Delta}_{h+1} (s_{h}) + \xi_{h} +  b_{h}(s_h) + e_{h}(s_h) + \Bar{e}_{h}(s_h)\\
         &= (1+\frac{1}{H})(\epsilon_{h} + \widetilde{\delta}_{h+1}) + \xi_{h} + b_{h}(s_h) + e_{h}(s_h) + \Bar{e}_{h}(s_h)
    \end{align*}
    where $\epsilon_{h} = P_h^{\pi}\widetilde{\Delta}_{h+1} (s_{h}) - \widetilde{\delta}_{h+1}$.\\
    Unrolling this recursion, we can have
    \begin{align*}
        &\widetilde{\delta}_{h} \leq \sum_{i=h}^{H-1} (1+\frac{1}{H})^i(\epsilon_{i} + \xi_{i} +  b_{i} + e_{i} + \Bar{e}_{i}) \\
        &\leq e \sum_{i=h}^{H-1}(\epsilon_{i} + \xi_{i} +  b_{i} + e_{i} + \Bar{e}_{i})
    \end{align*} 
    The last inequality is based on the fact that $(1+\frac{1}{H})^H \leq e$.
\end{proof}
\begin{lemma}\label{lemma_martingale}
    Let $k \in [K]$ and $h \in [H]$. Let the events $\calE$ and $\Omega_{k, h}$ hold. Then the following hold for every $s \in \calS$
    \begin{align*}
        \sum_{i=1}^{k}\sum_{j=h}^{H-1} \epsilon_{i, j} \leq H\sqrt{(H-h)k\iota} \leq H\sqrt{T_k\iota}
    \end{align*}
    Also the following bounds hold for every $s \in \calS$ and $h\in [H]$
    \begin{align*}
        \sum_{i=1}^{k}\indicator(s_{i, h}=s)\sum_{j=h}^{H-1} \epsilon_{i, j} \leq H\sqrt{(H-h)N'_{k, h}(o)\iota}
    \end{align*}
\end{lemma}
\begin{proof}
The fact that the event $\calE$ holds implies that the event $\calE_{\text{az}}(\calF_{\widetilde{\Delta},k,h}, H)$, $\calE_{\text{az}} (\calF'_{\widetilde{\Delta},k,h}, 1/\sqrt{\iota})$, $\calE_{\text{az}}(\calF_{\widetilde{\Delta},k, h, s}, H)$ and $\calE_{\text{az}}(\calF'_{\widetilde{\Delta},k, h, s}, 1/\sqrt{\iota})$ hold. Combined with the fact that $T_k \geq (H-h)k$, the first argument is proved.
For the second argument, with the event $\calE$, we can have
\begin{align*}
    \sum_{i=1}^{k}\indicator(s_{i, h}=s)\sum_{j=h}^{H-1} \epsilon_{i, j} \leq H\sqrt{(H-h)N'_{k, h}(s)\iota}   
\end{align*}
It is easy to verify that $\forall s = (q, o) \in \calS$, $\forall k\in [K], h \in [H], N'_{k, h}(s) \leq N'_{k, h}(o)$. Combined with this inequality, we complete the proof of the second argument.
\end{proof}
\begin{lemma}
    Let $k \in [K]$ and $h \in [H]$. Let the events $\calE$ and $\Omega_{k, h}$ hold. Then the following hold for every $s \in \calS$
    \begin{align*}
        \sum_{i=1}^{k-1}\sum_{j=h}^H \delta_{i, j} &\leq \sum_{i=1}^{k-1} \sum_{j=h}^H \widetilde{\delta}_{i, j} \leq HU_{k, 1}\\
        \sum_{i=1}^{k-1}\indicator(s_{i, h} = s)\sum_{j=h}^H \delta_{i, j} &\leq \sum_{i=1}^{k-1}\indicator(s_{i, h} = s) \sum_{j=h}^H \widetilde{\delta}_{i, j} \leq HU_{k, 1, s}
    \end{align*}
\end{lemma}
\begin{proof}
    By definition of $U_{k, h}$, the results of lemma \ref{lemma_total_regret} and lemma \ref{lemma_martingale}, we can get
    \begin{align*}
        \sum_{i=1}^k \delta_{i, h} &\leq \sum_{i=1}^k\widetilde{\delta}_{i,h} \leq U_{k,h} \leq U_{k,1}\\
        \sum_{i=1}^k \indicator(s_{i, h} = s) \delta_{i, h} &\leq \sum_{i=1}^k\indicator(s_{i, h} = s)\widetilde{\delta}_{i,h} \leq U_{k, h, s} \leq U_{k, 1, s}
    \end{align*}
    Consequently, we can have
    \begin{align*}
        \sum_{i=1}^{k-1}\sum_{j=h}^H \delta_{i, j} &\leq \sum_{i=1}^{k-1} \sum_{j=h}^H \widetilde{\delta}_{i, j} \leq HU_{k, 1}\\
        \sum_{i=1}^{k-1}\indicator(s_{i, h} = s)\sum_{j=h}^H \delta_{i, j} &\leq \sum_{i=1}^{k-1}\indicator(s_{i, h} = s) \sum_{j=h}^H \widetilde{\delta}_{i, j} \leq HU_{k, 1, s}
    \end{align*}
    Proof is completed.
\end{proof}
\begin{lemma}\label{lemma_var_pi}
    Let $k\in[K]$ and $h\in[H]$. Then under the events $\calE$ and $\Omega_{k, h}$ the following hold
    \begin{align*}
        \sum_{i=1}^k\sum_{j=h}^{H-1} \mathbbW_{i, j+1}^{\pi} &\leq T_k H + 2\sqrt{H^4T_k\iota}+\frac{4H^3\iota}{3}\\
        \sum_{i=1}^k\indicator(s_{i,h} = s)\sum_{j=h}^{H-1} \mathbbW_{i, j+1}^{\pi} &\leq N'_{k,h}(o) H^2 + 2\sqrt{H^5N'_{k,h}(o)\iota}+\frac{4H^3\iota}{3}
    \end{align*}
    \begin{proof}
        Under the event $\calE$, $\calE_{\text{fr}}(\calG_{\mathbbW,k,h}, H^4 T, H^3)$ and $\calE_{\text{fr}}(\calG_{\mathbbW,k,h,s}, H^5N'_{k,h}(s), H^3)$ hold, hence we have
        \begin{align*}
        \sum_{i=1}^k\sum_{j=h}^{H-1} \mathbbW_{i, j+1}^{\pi} &\leq \sum_{i=1}^k\mathbbE\left(\sum_{j=h}^{H-1} \mathbbW_{i, j+1}^{\pi}\,\middle\vert\,\calH_{i, h}\right) + 2\sqrt{H^4T_k\iota}+\frac{4H^3\iota}{3}\\
        \sum_{i=1}^k\indicator(s_{i,h} = s)\sum_{j=h}^{H-1} \mathbbW_{i, j+1}^{\pi} & \leq \sum_{i=1}^k\indicator(s_{i,h} = s)\mathbbE\left(\sum_{j=h}^{H-1} \mathbbW_{i, j+1}^{\pi}\,\middle\vert\,\calH_{i, h}\right) + 2\sqrt{H^5N'_{k,h}(s)\iota}+\frac{4H^3\iota}{3}
        \end{align*}
        The law of total variance leads to
        \begin{align*}
            \sum_{i=1}^k\mathbbE\left(\sum_{j=h}^{H-1} \mathbbV_{i, j+1}^{\pi}\,\middle\vert\,\calH_{i, h}\right) &= \sum_{i=1}^k\Var\left(\sum_{j=h+1}^H r^{\pi}_{i,j}\middle\vert\,\calH_{i, h}\right) \leq KH^2 = TH\\
            \sum_{i=1}^k\indicator(s_{i,h} = s)\mathbbE\left(\sum_{j=h}^{H-1} \mathbbV_{i, j+1}^{\pi}\,\middle\vert\,\calH_{i, h}\right) &= \sum_{i=1}^k\indicator(s_{i,h} = s)\Var\left(\sum_{j=h+1}^H r^{\pi}_{i,j}\middle\vert\,\calH_{i, h}\right) \leq N'_{k,h}(s)H^2            
        \end{align*}
        Combining with the fact that $N'_{k, h}(s) \leq N'_{k, h}(o)$ and lemma \ref{lemma_var}, we complete our proof.
    \end{proof}
\end{lemma}
\begin{lemma}\label{lemma_var_diff_1}
    Let $k\in[K]$ and $h\in[H]$. Then under the events $\calE$ and $\Omega_{k, h}$ the following hold
    \begin{align*}
        \sum_{i=1}^k \sum_{j=h}^{H-1} \left(\mathbbW_{i, j+1}^* - \mathbbW_{i, j+1}^{\pi}\right) &\leq 2H^2 U_{k, h} + 4H^2\sqrt{T_k\iota}\\
        \sum_{i=1}^k \indicator(s_{i, h} = s)\sum_{j=h}^{H-1} \left(\mathbbW_{i, j+1}^* - \mathbbW_{i, j+1}^{\pi}\right) &\leq 2H^2 U_{k, h, s} + 4H^2\sqrt{HN'_{k,h}(o)\iota}
    \end{align*}
\end{lemma}
\begin{proof}
By definition, we have
    \begin{align*}
        &\sum_{i=1}^k \sum_{j=h}^{H-1} \left(\mathbbW_{i, j+1}^* - \mathbbW_{i, j+1}^{\pi}\right) \\
        &= \sum_{i=1}^k \sum_{j=h}^{H-1} \Var_{o'\sim p(\cdot|o,a)} \left( W_{j+1}^* (q, o, a, o')\right)-\sum_{i=1}^k \sum_{j=h}^{H-1}\Var_{o'\sim p(\cdot|o,a)} \left( W_{j+1}^{\pi} (q,o,a, o')\right)
    \end{align*}
    For a better presentation, we make some short-hand notation: we denote $p_{i, j} = p(\cdot|o, a)$, $W_{i, j+1}^*(o') = W_{j+1}^* (q,o,a, o')$, $W_{i, j+1}^{\pi}(o') = W_{j+1}^{\pi} (q,o,a, o')$ and $n_{i, j} = N_i(o, a)$. Further
    \begin{align*}
        \sum_{i=1}^k \sum_{j=h}^{H-1} \left(\mathbbW_{i, j+1}^* - \mathbbW_{i, j+1}^{\pi}\right) &\leq \sum_{i=1}^k \sum_{j=h}^{H-1} \mathbbE_{o'\sim p_{i, j}}\left[(W^*_{i,j+1}(o'))^2 - (W^{\pi}_{i,j+1}(o'))^2\right]\\
        &= \sum_{i=1}^k \sum_{j=h}^{H-1} \mathbbE_{o'\sim p_{i, j}} \left[W^*_{i,j+1}(o') - W^{\pi}_{i,j+1}(o')\right]\left[W^*_{i,j+1}(o') + W^{\pi}_{i,j+1}(o')\right]\\
        &\leq 2H \underbrace{\sum_{i=1}^k \sum_{j=h}^{H-1}\mathbbE_{o'\sim p_{i, j}} \left[W^*_{i,j+1}(o') - W^{\pi}_{i,j+1}(o')\right]}_{(a)}
    \end{align*}
    The first inequality is derived from the definition of variance and the conditions \(V_{i,j}^* \geq V_{i,j}^\pi\), which implies \(W_{i,j}^* \geq W_{i,j}^\pi\).\\
    Using the same argument, we can also have the following:
    \begin{align*}
        \sum_{i=1}^k \indicator(s_{i, h} = s)\sum_{j=h}^{H-1} \left(\mathbbW_{i, j+1}^* - \mathbbW_{i, j+1}^{\pi}\right) &\leq 2H\underbrace{\sum_{i=1}^k\indicator(s_{i, h} = s) \sum_{j=h}^{H-1}\mathbbE_{o'\sim p_{i, j}} \left[W^*_{i,j+1}(o') - W^{\pi}_{i,j+1}(o')\right]}_{(b)}
    \end{align*}
    Under the event $\calE$, the event $\calE_{\text{az}}\left(\calF_{\widetilde{\Delta}, k, h}, H\right)$ holds. Plus, under event $\Omega_{k, h}$, we have $\delta_{k, h}\leq \widetilde{\delta}_{k, h}$. These combined can be used to prove that:
    \begin{align*}
        (a) \leq \sum_{i=1}^k \sum_{j=h}^{H-1} \widetilde{\delta}_{i, j+1} + 2H\sqrt{T_k\iota}\leq H U_{i, h} + 2H\sqrt{T_k\iota}
    \end{align*}
    Under the event $\calE$, the event $\calE_{\text{az}}\left(\calF_{\widetilde{\Delta}, k, h, s}, H\right)$ holds, and under the event $\Omega_{k, h}$, we can bound $(b)$ with:
    \begin{align*}
        (b) \leq  \sum_{i=1}^k\indicator(s_{i,h} = s)\sum_{j=h}^{H-1} \widetilde{\Delta}_{i, j+1} + 2H^{1.5}\sqrt{N'_{k,h}(s)\iota} \leq  HU_{k,h,s}+ 2H^{1.5}\sqrt{N'_{k,h}(o)\iota}
    \end{align*}
    The last inequality is based on the fact that $N'_{k, h}(s) \leq N'_{k, h}(o)$. We complete our proof by multiplying $(a)$ and $b$ by $2H$.  
\end{proof}
\begin{lemma}\label{lemma_var_diff_2}
    Let $k\in[K]$ and $h\in[H]$. Then under the events $\calE$ and $\Omega_{k, h}$ the following hold
    \begin{align*}
        \sum_{i=1}^k \sum_{j=h}^{H-1} \left(\widehat{\mathbbW}_{i, j+1} - \mathbbW_{i, j+1}^{\pi}\right) &\leq 2H^2 U_{k, h} + 15H^2O\sqrt{AT_k\iota}\\
        \sum_{i=1}^k \indicator(s_{i, h} = s)\sum_{j=h}^{H-1} \left(\widehat{\mathbbW}_{i, j+1} - \mathbbW_{i, j+1}^{\pi}\right) &\leq 2H^2 U_{k, h, s} + 15H^2O\sqrt{HAN'_{k,h}(o)\iota}
    \end{align*}
\end{lemma}
\begin{proof}
We use the same notation as lemma \ref{lemma_var_diff_1}. We denote $p_{i, j} = p(\cdot|o, a)$, $W_{i, j+1}^*(o') = W_{j+1}^* (q,o,a, o')$, $W_{i, j+1}^{\pi}(o') = W_{j+1}^{\pi} (q,o,a, o')$ and $n_{i, j} = N_{i,j}(o, a)$. We only need to prove the first argument, since the second inequality can be proved in a similar manner. The only difference is that $T_k$ and $U_{k, h}$ are replaced by $HN'_{k,h}(o)$ and $U_{k, h, s}$, respectively. We start by proving the first inequality, and the following inequalities hold:
    \begin{align*}
        \sum_{i=1}^k \sum_{j=h}^{H-1} \left(\widehat{\mathbbW}_{i, j+1} - \mathbbW_{i, j+1}^{\pi}\right) &\leq \sum_{i=1}^k \sum_{j=h}^{H-1} \mathbbE_{o'\sim \widehat{p}_{i, j}}\left(W_{i, j+1}(o')\right)^2 - \mathbbE_{o'\sim p_{i, j}}\left(W_{j+1}^{\pi_i}(o')\right)^2\\
        &+ \sum_{i=1}^k \sum_{j=h}^{H-1} \left(\mathbbE_{o'\sim p_{i, j}}W_{j+1}^*(o')\right)^2 - \left(\mathbbE_{o'\sim \widehat{p}_{i, j}}W_{j+1}^*(o')\right)^2\\
        &\leq \underbrace{\sum_{i=1}^k \sum_{j=h}^{H-1} \mathbbE_{o'\sim \widehat{p}_{i, j}}\left(W_{i, j+1}(o')\right)^2 - \sum_{i=1}^k \sum_{j=h}^{H-1} \mathbbE_{o'\sim p_{i, j}}\left(W_{i, j+1}(o')\right)^2}_{(a)}\\
        &+ \underbrace{\sum_{i=1}^k \sum_{j=h}^{H-1}\mathbbE_{o'\sim p_{i, j}}\left[\left(W_{i, j+1}(o')\right)^2 - \left(W_{j+1}^{\pi_i}(o')\right)^2\right]}_{(b)}\\
        &+\underbrace{\sum_{i=1}^k \sum_{j=h}^{H-1}4H^2\sqrt{\frac{\iota}{n_{i, j}}}}_{(c)}
    \end{align*}
    The first inequality holds because under $\Omega_{k, h}$, $V_{i, j} \geq V_j^* \geq V_j^{\pi_i}$ and consequently, $W_{i, j} \geq W_j^* \geq W_j^{\pi_i}$. The second inequality holds by enlarging confidence interval under event $\calE$.
    \begin{align*}
        (a) &\leq \sum_{i=1}^k \sum_{j=h}^{H-1} 2H^2\sqrt{\frac{O\iota}{n_{i, j}}}\\
            &\leq 5H^2O\sqrt{AT_k\iota}
    \end{align*}
    The first inequality holds for the event $\calE$. The second inequality holds because of the pigeon-hole argument(see \eg Appendix C.3 of \citet{auer2008near}).
    \begin{align*}
        (b) \leq 2H\left(\sum_{i=1}^k \sum_{j=h}^{H-1} \widetilde{\delta}_{i, j+1} + 2H\sqrt{T_k\iota}\right) \leq 2H^2 \left(U_{k, 1}+ 2\sqrt{T_k\iota}\right)
    \end{align*}
    The first inequality hold by using the same argument of lemma \ref{lemma_var_diff_1}.
    \begin{align*}
        (c) \leq 6H^2\sqrt{OAT_k\iota}
    \end{align*}
    By applying another pigeon-hole principle.
\end{proof}
\begin{lemma}\label{lemma_est_err}
    Let $k\in[K]$ and $h\in[H]$. Then under the events $\calE$ and $\Omega_{k, h}$ the following hold
    \begin{align*}
        C_{k, h} &\leq 4\sqrt{HOAT_k} + 4\sqrt{H^2U_{k,1}OA\iota^2}\\
        C_{k, h, s} &\leq 4\sqrt{H^2OAN'_{k, h}(o)} + 4\sqrt{H^2U_{k, 1, s}OA\iota^2}
    \end{align*}
\end{lemma}
\begin{proof}
We only need to prove the first argument, since the second inequality can be proved in a similar manner. The only difference is that $T_k$ and $U_{k, h}$ are replaced by $HN'_{k,h}(o)$ and $U_{k, h, s}$, respectively. Hence, we start by proving the first argument:
    \begin{align*}
        C_{k, h} &= \sum_{i=1}^k \indicator(i\in [k]_{\text{typ}})\sum_{j=h}^{H-1} \left(2\sqrt{\frac{\mathbbW^*_{i, j+1}\iota}{n_{i, j}}} + \frac{14H\iota}{3n_{i, j}}\right)\\
        &\leq 2\sqrt{\iota} \sqrt{\underbrace{\sum_{i=1}^k \sum_{j=h}^{H-1} \mathbbW^*_{i, j+1}}_{(a)}} \sqrt{\underbrace{\sum_{i=1}^k \indicator(i\in [k]_{\text{typ}}) \sum_{j=h}^{H-1} \frac{1}{n_{i, j}}}_{(b)}} + \sum_{i=1}^k \indicator(i\in [k]_{\text{typ}})\sum_{j=h}^{H-1}\frac{14H\iota}{3n_{i,j}}
    \end{align*}
    \begin{align*}
        (a) = \underbrace{\sum_{i=1}^k \sum_{j=h}^{H-1} \mathbbW^{\pi}_{i, j+1}}_{(c)} + \underbrace{\sum_{i=1}^k \sum_{j=h}^{H-1} \mathbbW^*_{i, j+1}-\mathbbW^{\pi}_{i, j+1}}_{(d)}
    \end{align*}
    $(c)$ and $(d)$ can be bounded under the events $\calE$ and $\Omega_{k, h}$ using the results of lemma \ref{lemma_var_pi} and lemma \ref{lemma_var_diff_1}. Hence,
    \begin{align*}
        (a) &\leq T_k H + 2H^2 U_{k, h} + \frac{4H^3\iota}{3} + 6H^2\sqrt{T_k\iota}\\
            &\leq 2T_k H + 2H^2 U_{k, h}
    \end{align*}
    The second inequality is based on the fact that $T_k \geq 650H^2O^2A^2\iota$.
    $(b)$ can be bounded by using pigeon-hole principle
    \begin{align*}
        (b) \leq \sum_{(o, a) \in \calO \times \calA}\sum_{n=1}^{N_k(o, a)} \frac{1}{n} \leq OA\sum_{n=1}^{T} \frac{1}{n} \leq 2OA\ln(T)
    \end{align*}
    Combining all these and the condition that $T_k \geq 650H^2O^2A^2\iota$, we complete the proof.
\end{proof}
\begin{lemma}\label{lemma_bonus}
    Let $k\in[K]$ and $h\in[H]$. Then under the events $\calE$ and $\Omega_{k, h}$ the following hold
    \begin{align*}
        B_{k, h} &\leq 11\iota\sqrt{HOAT_k} + 8H\iota\sqrt{OAU_{k,1}} + 580HO^2A^{\frac{3}{2}}\iota^2\\
        B_{k, h, s} &\leq 11\iota\sqrt{H^2OAN'_{k, h}(o)} + 8H\iota\sqrt{OAU_{k,1, s}} + 580HO^2A^{\frac{3}{2}}\iota^2   
    \end{align*}
\end{lemma}
\begin{proof}
We only need to prove the first argument, since the second inequality can be proved in a similar manner. The only difference is that $T_k$ and $U_{k, h}$ are replaced by $HN'_{k,h}(o)$ and $U_{k, h, s}$, respectively. Hence, we start by proving the first argument:
    \begin{align*}
        B_{k,h} &= \underbrace{\sum_{i=1}^{k} \mathbb{I}(i \in [k]_{\text{typ}}) \sum_{j=h}^{H-1} \sqrt{\frac{8 \widehat{\mathbbW}_{i,j+1} \iota}{n_{i,j}}}}_{(a)} + \sum_{i=1}^k \indicator(i\in [k]_{\text{typ}})\sum_{j=h}^{H-1}\left(\frac{14H\iota}{3n_{i,j}}+\sqrt{\frac{\iota}{n_{i, j}}}\right)\\
        &+ \underbrace{\sum_{i=1}^{k} \mathbb{I}(i \in [k]_{\text{typ}}) \sum_{j=h}^{H-1} \left( \sqrt{\frac{8}{n_{i,j}} \sum_{o' \in \mathcal{O}} \widehat{p}_{i,j}(o') \min \left( \frac{100^2 O^2 H^2 AL^2}{N'_{i,j+1}(o')}, H^2 \right)} \right)}_{(b)}
    \end{align*}
    We can use similar technique to lemma \ref{lemma_est_err} to bound $(a)$.
    \begin{align*}
        (a) \leq \sqrt{8\iota \underbrace{\sum_{i=1}^{k} \sum_{j=h}^{H-1} \widehat{\mathbbW}_{i,j+1}(q, o, a)}_{(c)} \underbrace{\sum_{i=1}^{k} \mathbb{I}(k \in [k]_{\text{typ}}) \sum_{j=h}^{H-1} \frac{1}{n_{i,j}}}_{(d)}}
    \end{align*}
    \begin{align*}
        (c) = \underbrace{\sum_{i=1}^k \sum_{j=h}^{H-1} \mathbbW^{\pi}_{i, j+1}}_{(e)} + \underbrace{\sum_{i=1}^k \sum_{j=h}^{H-1} \widehat{\mathbbW}_{i, j+1}-\mathbbW^{\pi}_{i, j+1}}_{(f)}
    \end{align*}
    $(e)$ and $(f)$ can be bounded under $\calE$ and $\Omega_{k, h}$ using lemma \ref{lemma_var_pi} and lemma \ref{lemma_var_diff_2}. Hence
    \begin{align*}
        (c) &\leq T_k H +\frac{4H^3\iota}{3} + 2H^2 U_{k, 1} + 15H^2O\sqrt{AT_k\iota}\\
        &\leq 2T_k H + 2H^2 U_{k, 1}
    \end{align*}
    The last inequality is based on the fact that $T_k \geq 650H^2O^2A^2\iota$. Applying pigeon-hole principle to $(d)$, $(a)$ is bounded by
    \begin{align*}
        (a) \leq 8\iota\sqrt{HOAT_k} + 8H\iota\sqrt{OAU_{k,1}}
    \end{align*}
    \begin{align*}
        (b) \leq \sqrt{8 \underbrace{\sum_{i=1}^{k} \indicator(i \in [k]_{\text{typ}}) \sum_{j=h}^{H-1} \sum_{o' \in \calO} \widehat{p}_{i,j}(o') b'_{i,j+1}(o')}_{(g)}\underbrace{\sum_{i=1}^{k} \mathbb{I}(k \in [k]_{\text{typ}}) \sum_{j=h}^{H-1} \frac{1}{n_{i,j}}}_{(h)} }
    \end{align*}
    \begin{align*}
        (g) \leq \underbrace{\sum_{i=1}^{k} \sum_{j=h}^{H-1} (\widehat{p}_{i,j} - p_{i,j}) b'_{i,j+1}}_{(i)} + \underbrace{\sum_{i=1}^{k} \sum_{j=h}^{H-1} (p_{i,j} b'_{i,j+1} - b'_{i,j+1}(o_{i,j+1}))}_{(j)} + \underbrace{\sum_{i=1}^{k} \sum_{j=h}^{H-1} b'_{i,j+1}(o_{i,j+1})}_{(k)}.
    \end{align*}
    $(i)$ can be bounded by $2\sqrt{2}H^2O\sqrt{T_k A \iota}$ by the pigeon-hole principle and using the concentration inequality under $\calE$. $(j)$ can be bounded by $2H^2\sqrt{T_k\iota}$ since it is sum of the martingale difference sequence. $(k)$ can be bounded by $20000H^2O^3A^2\iota^3$ by using the pigeon-hole principle. Hence
    \begin{align*}
        (b) &\leq \sqrt{32\sqrt{2}H^2O^2\sqrt{T_k A^3\iota^3}+ 32H^2OA\sqrt{T_k\iota} +320000H^2O^4A^3\iota^4} \\
        &\leq 3\iota\sqrt{HOAT_k} + 570HO^2A^{\frac{3}{2}}\iota^2
    \end{align*}
    The last inequality is based on the fact that $T_k \geq 650H^2O^2A^2\iota$. Combining all these, we complete our proof.
\end{proof}
\begin{lemma}\label{lem:final_regret}
    Under the events $\calE$ and $\Omega_{K, 1}$, the following hold
    \begin{align*}
        \Regret(K) \leq \widetilde{\Regret}(K) \leq U_{K, 1} \leq 72\iota\sqrt{HOAT} + 2500 H^2 O^2 A^{\frac{3}{2}}\iota^2 + 4H\sqrt{T\iota} 
    \end{align*}
\end{lemma}
\begin{proof}
 Under the event $\calE$, $C'_{K, 1}$ can be bounded by $3\sqrt{OAT\iota}$ using the pigeon-hole principle. Similarly, $\sum_{i=1}^K \sum_{j=1}^H \xi_{i, j}$ can be bounded by $4H^2O^2A\iota^2$ using pigeon-hole principle. Then, we sum up the regret due to $B_{k, h}$, $C_{k, h}$ and $\epsilon_{k, h}$ from lemma \ref{lemma_bonus}, lemma \ref{lemma_est_err} and lemma \ref{lemma_martingale} to bound $U_{K, 1}$. The following holds:
 \begin{align*}
     \widetilde{\Regret}(K) \leq U_{K, 1} \leq 18\iota\sqrt{HOAT} + 12H\iota\sqrt{OAU_{K,1}} + 590 H^2 O^2 A^{\frac{3}{2}}\iota^2 + H\sqrt{T\iota}     
 \end{align*}
 By solving the bound in terms of $U_{k, 1}$, we complete our proof.
\end{proof}
\begin{lemma}\label{lemma_regret_khx}
    Let $k\in[K]$ and $h\in[H]$. Then under the events $\calE$ and $\Omega_{k, h}$ the following hold
    \begin{align*}
        \Regret(k,s,h) &\leq \widetilde{\Regret}(k, s, h) \leq 72H\iota\sqrt{OAN'_{k, h}(o)} + 2500H^2O^2A^{\frac{3}{2}}\iota^2 + 4H^{\frac{3}{2}}\sqrt{N'_{k, h}(o)\iota}\\
        &\leq 100H^{\frac{3}{2}}O\iota\sqrt{AN'_{k, h}(o)}
    \end{align*}
\end{lemma}
\begin{proof}
    The proof is similar to the proof of lemma \ref{lem:final_regret}, we start by getting a equation in terms of $U_{k, s, h}$ by summing up regret due to $B_{k, h, s}$, $C_{k, s, h}$, $C'_{k, h}$ and $\xi$. Then we solve the bound in term of $U_{k, s, h}$ to get our result.
\end{proof}
\begin{lemma}\label{lemma_regret_state}
    Let $k\in[K]$ and $h\in[H]$. Then under the events $\calE$ and $\Omega_{k, h}$ the following hold for every $(q, o,a,o') \in \calQ\times\calO\times\calA\times\calO$
    \begin{align*}
        W_{k,h}(q,o,a,o') - W_h^*(q,o,a,o') \leq 100\sqrt{\frac{H^3O^2A\iota^2}{N'_{k, h}(o')}}
    \end{align*}
\end{lemma}
\begin{proof}
    \begin{align*}
        &\quad \sum_{i=1}^k \indicator(o_{i,h} = o') \left(W_{i,h}(q,o,a,o') - W_h^*(q,o,a,o')\right)\\
        &= \sum_{i=1}^k\indicator(o_{i,h} = o') \sum_{q'\in\calQ} \tau(q'|q,L(o,a,o'))\left(V_{i, h}(q', o') - V_h^*(q', o')\right)\\
        &= \sum_{q'\in\calQ} \sum_{i=1}^k \indicator(o_{i,h} = o') \tau(q'|q,L(o,a,o'))\left(V_{i, h}(q', o') - V_h^*(q', o')\right)\\
        &\leq \sum_{q'\in\calQ} \sum_{i=1}^k \indicator(o_{i,h} = o') \tau(q'|q,L(o,a,o'))\left(V_{i, h}(q', o') - V_h^{\pi_i}(q', o')\right)\\
        &= \sum_{q'\in\calQ} \tau(q'|q,L(o,a,o')) \sum_{i=1}^k\indicator(s_{i, h} = (q', o'))\left(V_{i, h}(q', o') - V_h^{\pi_i}(q', o')\right)\\
        &\leq \sum_{q'\in\calQ} \tau(q'|q,L(o,a,o')) 100H^{\frac{3}{2}}O\iota\sqrt{AN'_{k, h}(o')} =100H^{\frac{3}{2}}O\iota\sqrt{AN'_{k, h}(o')}
    \end{align*}
    The first inequality is based on the fact that $V_h^*\geq V_h^{\pi}$. The second inequality holds per lemma \ref{lemma_regret_khx}. Since $V_{k, h}$ by definition is monotonically non-increasing in $k$, $W_{k, h}$ is monotonically non-increasing in $k$ as well. Then we have
    \begin{align*}
        N'_{k, h}(o')\left(W_{k,h}(q,o,a,o') - W_h^*(q,o,a,o')\right) \leq 100H^{\frac{3}{2}}O\iota\sqrt{AN'_{k, h}(o')}
    \end{align*}
    We complete our proof by dividing $N'_{k,h}(o')$ on both sides of the above inequality.
\end{proof}
\begin{lemma}\label{lem:Omega}
    Under the event $\calE$, the set of events $\{\Omega\}_{k\in[K], h\in [H]}$ hold.
\end{lemma}
\begin{proof}
    We prove this by induction. First, $V_{1, h} = H \geq V_h^*$. And $V_{k, H+1} = V^*_{H+1} = 0$. \\
    To prove this result, we need to show if $\Omega_{k, h+1}$ holds then $\Omega_{k, h}$ also holds. If $\Omega_{k, h+1}$ holds, we have following hold for every $(i, j) \in [k, h]_{\text{hist}}$ per lemma \ref{lemma_regret_state}.
    \begin{align*}
        W_{k,h+1}(q,o,a,o') - W_{h+1}^*(q,o,a,o')\leq 100\sqrt{\frac{H^3O^2A\iota^2}{N'_{k, h+1}(o')}}
    \end{align*}
    Per the algorithm, we have
    \begin{align*}
        V_{k, h} = \min\{V_{k-1, h}, H, \widehat{R}(s, \pi_k(s, h))+b_{k, h}(s, \pi_k(s, h))+\widehat{P}_h^{\pi_k}V_{i, j+1}(s)\}
    \end{align*}
    If $V_{k, h} = V_{k-1, h}$, $V_{k, h} \geq V_h^*$ holds trivially by invoking induction. Also when $V_{k, h} = H$, $\Omega_{k, h}$ holds trivially. Thus, we only need to consider when $V_{k, h} = \widehat{R}(s, \pi_k(s, h))+b_{k, h}(s, \pi_k(s, h))+\widehat{P}_h^{\pi_k}V_{i, j+1}(s)$. In this case,
    \begin{align*}
        &V_{k,h}(s) - V_h^*(s) \\
        &\geq \widehat{R}(s, \pi^*(s, h))+b_{k, h}(s, \pi^*(s, h))+\widehat{P}_h^{\pi^*}V_{i, j+1}(s) - R(s, \pi^*(s, h)) - P_h^{\pi^*}V_{h+1}^*(s)\\
        &= b_{k, h}(s, \pi^*(s, h)) + \left(\widehat{R}(s, \pi^*(s, h))-R(s, \pi^*(s, h))\right) + \left(\widehat{P}_h^{\pi^*} - P_h^{\pi^*}\right)V_{h+1}^*(s) + \widehat{P}_h^{\pi^*} \left(V_{i, h+1} - V_{h+1}^*\right)(s)
    \end{align*}
    The inequality holds based on the fact that $\pi_{k, h}$ is the optimal policy w.r.t. $V_{k,h}$. From the induction assumption, $\Omega_{k, h+1}$ holds, consequently, $V_{i, h+1} \geq V_{h+1}^*$. Then we have
    \begin{align*}
        &V_{k,h}(s) - V_h^*(s) \geq b_{k, h}(s, \pi^*(s, h)) + \left(\widehat{P}_h^{\pi^*} - P_h^{\pi^*}\right)V_{h+1}^*(s) + \left(\widehat{R}(s, \pi^*(s, h))-R(s, \pi^*(s, h))\right)\\
        &\geq b_{k, h} - c_1(\widehat{\mathbbW}_{k, h}^*(s, a), H, N_k) - c_4(1, N_k)\\
        &\geq \underbrace{\sqrt{\frac{8\widehat{\mathbbW}_{k, h}\iota}{N_k}} - 2\sqrt{\frac{\widehat{\mathbbW}_{h}^*\iota}{N_k}}}_{(a)} - \frac{14\iota}{3N_k} - \sqrt{\frac{\iota}{N_k}}\\
        &+ \sqrt{\frac{8}{N_k} \widehat{p}_k \min \left( \frac{100^2 O^2 H^2 AL^2}{N'_{k,h+1}(o')}, H^2 \right)} +  \frac{14\iota}{3N_k} + \sqrt{\frac{\iota}{N_k}}
    \end{align*}
Further, for $(a)$ we have
\begin{align*}
    (a) \geq
\begin{dcases} 
      0 & \text{if } 2 \widehat{\mathbbW}_{k, h} \geq \widehat{\mathbbW}^*_{h} \\
      -\sqrt{\frac{4\widehat{\mathbbW}^*_{h} - 8 \widehat{\mathbbW}_{k, h}}{N_k}} & \text{otherwise}
\end{dcases}
\end{align*}
By leveraging lemma \ref{lemma_var_ineq}, we have
\begin{align*}
    \widehat{\mathbbW}^*_{h} \leq 2 \widehat{\mathbbW}_{k, h} + 2\Var_{o'\sim \widehat{p}_k}\left(W_{k, h+1}(o') - W^*_{h+1}(o')\right) \leq 2 \widehat{\mathbbW}_{k, h} + 2\underbrace{\widehat{p}_k\left(W_{k, h+1} - W^*_{h+1}\right)^2}_{(b)}
\end{align*}
For $(b)$, we have
\begin{align*}
    (b) &= \sum_{o'\in \calO} \widehat{p}_k(o'|o, a)\left(W_{k, h+1}(q, o, a, o') - W^*_{h+1}(q, o, a, o')\right)^2\\
    &\leq \sum_{o'\in \calO} \widehat{p}_k(o'|o, a) \left(100\sqrt{\frac{H^3O^2A\iota^2}{N'_{k, h+1}(o')}}\right)^2
\end{align*}
This inequality is based on lemma \ref{lemma_regret_state}. Hence $(a)$ can be lower bounded by 
\begin{align*}
    (a) \geq
    \begin{dcases} 
      0 & \text{if } 2 \widehat{\mathbbW}_{k, h} \geq \widehat{\mathbbW}^*_{h} \\
      -\sqrt{\frac{8}{N_k} \widehat{p}_k \min \left( \frac{100^2 O^2 H^3 A\iota^2}{N'_{k,h+1}(o')}, H^2 \right)} & \text{otherwise}
    \end{dcases}
\end{align*}
Combining all these proves $V_{k, h} \geq V_h^*$. Thus, the event $\Omega_{k, h}$ holds. The proof is completed by invoking induction from $H$ to $1$.
\end{proof}
\subsection{Proof of Theorem \ref{thm:regret}}
Combining lemma \ref{lem:calE}, lemma \ref{lem:Omega} and lemma \ref{lem:final_regret}, we can complete the proof of theorem \ref{thm:regret}.

\section{Exploration with Non-Markovian Rewards}\label{appendix:rf}

\begin{lemma}\label{lemma_rf_helper}
Suppose $\widehat{p}$ is the emprical transition matrix formed by sampling according to $\lambda$ distribtution for $N$ samples, and when $N \geq \frac{8O^3A^3H^2}{\delta^2}\ln{\left(\frac{2H}{\rho}\right)}$, then w.p. at least $1-2\rho$,
    \begin{align*}
        \sum_{o\in O_h^{\delta}, a} \left(\sum_{o'} \epsilon(o'|o, a)\right)^2\lambda^{\pi}(o, a)\leq \frac{8O^2A}{N}\ln{\left(\frac{2OAH}{\rho}\right)}
    \end{align*}
where $\epsilon_h (o'|o, a) = \left|p_h(o'|o, a) - \widehat{p}_h(o'|o,a)\right|$.
\end{lemma}
\begin{proof}
    By Azuma-Hoeffding's inequality, w.p. at least $1-\rho$ we can have:
    \begin{align*}
      \left(\sum_{o'} \epsilon (o'|o, a)\right)^2  \leq \frac{4O}{N^{\lambda}(o, a)}\ln{\left(\frac{2OAH}{\rho}\right)}
    \end{align*}
    $N^{\lambda}(s, a)$ is the number of visits to state action pair $(o, a)$ under the distribution $\lambda$.\\
    With Hoeffding's inequality, w.p. at least $1-\rho$, we can have:
    \begin{align*}
        \left|\frac{N^{\lambda}(o, a)}{N} - \lambda(o, a)\right| \leq \sqrt{\frac{2OA}{N} \ln{\left(\frac{2H}{\rho}\right)}}
    \end{align*}
    which gives,  w.p. at least $1-\rho$:
    \begin{align*}
        N^{\lambda}(o, a) \geq N\lambda(o, a) - \sqrt{2OAN \ln{\left(\frac{2H}{\rho}\right)}}
    \end{align*}
    Hence, we have:
    \begin{align*}
        &\quad \sum_{o\in O^{\delta}, a} \left(\sum_{o'} \epsilon (o'|o, a)\right)^2\lambda(o, a) \\
        &\leq \sum_{o\in O^{\delta}, a} \frac{4O}{N\lambda(o, a) - \sqrt{2OAN \ln{\left(\frac{2H}{p}\right)}}}\ln{\left(\frac{2OAH}{\rho}\right)} \lambda(o, a)\\
        &\leq 4O\ln{\left(\frac{2OAH}{\rho}\right)} \underbrace{\sum_{o\in O^{\delta}, a} \frac{\lambda(o,a)}{N\lambda(o, a) - \sqrt{2OAN \ln{\left(\frac{2H}{\rho}\right)}}}}_{(a)}
    \end{align*}
    By definition, $\forall o\in \calO^{
    \delta}, \lambda(o,a) \geq \frac{\delta}{2OAH}$,
    which leads to,
    \begin{align*}
        (a) \leq OA \frac{\delta}{N\delta - 2OAH\sqrt{2OAN \ln{\left(\frac{2H}{\rho}\right)}}}
    \end{align*}
    The inequality is based on the fact that there is only $OA$ state action pairs in total. When $N \geq \frac{8O^3A^3H^2}{\delta^2}\ln{\left(\frac{2H}{p}\right)}$, we can have:
    \begin{align*}
        (a) \leq \frac{2OA}{N}
    \end{align*}
    Combining all these, we prove our result.
\end{proof}
\subsection{Proof of Lemma \ref{lemma_simulation}}\label{prf:lemma_simulation}
\begin{proof}
We define $F(\eta)$ as the expected reward collected by trajectory $\eta$, then for every $\pi$,
\begin{align*}
    J(\pi) = \sum_{\eta} F(\eta)p(\eta;\pi)
\end{align*}
where $p(\eta;\pi) =p(o_1) \prod_{t=1}^{H} \pi(a_t|o_t)\prod_{t=1}^{H}p(o_{t+1}|o_t, a_t)$. Note that, for PRMs, $F(\eta)$ represents the expectation of the reward of trajectory $\eta$. However, for DRMs, $F(\eta)$ is a deterministic quantity. Further, we have
\begin{align*}
    \left|\widehat{J}(\pi) - J(\pi)\right| = \left|\sum_{\eta}   F(\eta) (\widehat{p}(\eta ; \pi) - p(\eta ; \pi))\right| \\
    \leq G \sum_{\eta}\left| (\widehat{p}(\eta ; \pi) - p(\eta ; \pi))\right|
\end{align*}
The inequality is based on Holder's inequality, and the definition of $p(\eta; \pi)$ leads to
\begin{align*}
    &\left| \widehat{J}(\pi) - J(\pi) \right| \leq \Maxr \sum_{\eta} p(o_1) \prod_{t=1}^{H} \pi(a_t|\eta_{1:t})\left|\underbrace{\prod_{t=1}^{H}\widehat{p}(o_{t+1}|o_t, a_t) - \prod_{t=1}^{H} p(o_{t+1}|o_t, a_t)}_{(a)}\right|\\
\end{align*}
where $(a)$ can be further rewritten as
\begin{align*}
&\prod_{t=1}^{H}\widehat{p}(o_{t+1}|o_t, a_t) - \prod_{t=1}^{H} p(o_{t+1}|o_t, a_t)\\
&= \prod_{t=1}^{H}\widehat{p}(o_{t+1}|o_t, a_t) - \prod_{t=1}^{H} p(o_{t+1}|o_t, a_t) \pm \sum_{m=2}^H \prod_{t=1}^{m-1}p(o_{t+1}|o_t, a_t) \prod_{t=m}^H \widehat{p}(o_{t+1}|o_t, a_t)\\
&=\sum_{t=1}^H (\widehat{p}(o_{m+1}|o_m, a_m) - p(o_{m+1}|o_m, a_m))\prod_{t=1}^{m-1} p(o_{t+1}|o_t, a_t) \prod_{t=m+1}^H \widehat{p}(o_{t+1}|o_t, a_t)
\end{align*}
Hence,
\begin{align*}
    \left|\prod_{t=1}^{H}\widehat{p}(o_{t+1}|o_t, a_t) - \prod_{t=1}^{H} p(o_{t+1}|o_t, a_t)\right|\leq \sum_{m=1}^{H} \epsilon(o_{m+1}|o_m, a_m) \prod_{t=1}^{m-1} p(o_{t+1}|o_t, a_t) \prod_{t=m+1}^H \widehat{p}(o_{t+1}|o_t, a_t)
\end{align*}
where $\epsilon(o_{m+1}|o_m, a_m) = \left|\widehat{p}(o_{m+1}|o_m, a_m) - p(o_{m+1}|o_m, a_m)\right|$. Consequently,
\begin{align*}
    &\left| \widehat{J}(\pi) - J(\pi) \right| \leq G\underbrace{\sum_{\eta} p(o_1) \prod_{t=1}^{H} \pi(a_t|\eta_{1:t}) \sum_{m=1}^{H} \epsilon(o_{m+1}|o_m, a_m) \prod_{t=1}^{m-1} p(o_{t+1}|o_t, a_t) \prod_{t=m+1}^H \widehat{p}(o_{t+1}|o_t, a_t)}_{(b)}
\end{align*}
\begin{align*}
    &(b)= \sum_{m=1}^{H} \sum_{o_m, a_m, o_{m+1}} \epsilon(o_{m+1} | o_m, a_m) \left(\sum_{\eta_{1:m-1}} \pi(a_m|\eta_{1:m-1}, o_m) p(o_1) \prod_{t=1}^{m-1}\pi(a_t|\eta_{1:t}) p(o_{t+1} | o_t, a_t)\cdot \right.\\
    &\left.\sum_{a_{m+1}}\sum_{\eta_{m+2:H}} \prod_{t=m+1}^{H}\pi(a_t|\eta_{1:m-1}, o_m, a_m, o_{m+1},\eta_{m+2:t}) \widehat{p}(o_{t+1} | o_t, a_t)\right)\\
    &= \sum_{m=1}^{H} \sum_{o_m, a_m, o_{m+1}} \epsilon(o_{m+1} | o_m, a_m) \left(\sum_{\eta_{1:m-1}} \pi(a_m|\eta_{1:m-1}, o_m) p(o_1) \prod_{t=1}^{m-1}\pi(a_t|\eta_{1:t}) p(o_{t+1} | o_t, a_t)\cdot\right.\\
    &\left. \sum_{o_H, a_H} \widehat{\mu}_{H}^{\pi}(o_H, a_H|\eta_{1:m-1}, o_m, a_m, o_{m+1})\right)\\
    & = \sum_{m=1}^{H} \sum_{o_m, a_m, o_{m+1}} \epsilon(o_{m+1} | o_m, a_m) \left(\sum_{\eta_{1:m-1}} \pi(a_m|\eta_{1:m-1}, o_m) p(o_1) \prod_{t=1}^{m-1}\pi(a_t|\eta_{1:t}) p(o_{t+1} | o_t, a_t)\right)\\
    &= \sum_{m=1}^{H} \sum_{o_m, a_m, o_{m+1}} \epsilon(o_{m+1} | o_m, a_m) \mu_m^\pi(o_m, a_m)
\end{align*}
Hence, by replacing $(b)$ with $\sum_{m=1}^{H} \sum_{o_m, a_m, o_{m+1}} \epsilon(o_{m+1} | o_m, a_m) \mu_m^\pi(o_m, a_m)$, we complete our proof.
\end{proof}

\subsection{Proof of Lemma \ref{lemma_policy_diff}}\label{prf:lemma_policy_diff}
\begin{proof}
    With lemma \ref{lemma_simulation}, we have
    \begin{align*}
    \left| \widehat{J}(\pi) - J(\pi) \right| \leq \sum_{m=1}^{H} \sum_{o_m, a_m, o_{m+1}} \epsilon(o_{m+1} | o_m, a_m) \mu_m^\pi(o_m, a_m) \Maxr
    \end{align*}
    We let $\calO^\delta \defeq \{o: \max_{\pi}\mu^{\pi}(o) \geq \delta\}$, we have,
    \begin{align*}
        &\sum_{m=1}^{H} \sum_{o_m, a_m, o_{m+1}} \epsilon(o_{m+1} | o_m, a_m) \mu_m^\pi(o_m, a_m) \Maxr\\
        &=\Maxr\left( \underbrace{\sum_{o'}\sum_{o\in O^{\delta}, a} \epsilon(o'|o, a) \sum_{h=1}^H \mu_h^{\pi}(o, a)}_{(a)} + \underbrace{\sum_{o'}\sum_{o\notin O^{\delta}, a}\epsilon(o'|o, a) \sum_{h=1}^H \mu_h^{\pi}(o, a)}_{(b)}\right)
    \end{align*}
    By definition of insignificant state, we have:
    \begin{align*}
        (b) \leq 2\sum_{o\notin O_h^{\delta}, a}\mu^{\pi}(o, a) = 2\sum_{o\notin O_h^{\delta}}\mu^{\pi}(o) \leq 2\sum_{o\notin O_h^{\delta}}\delta \leq 2O\delta
    \end{align*}
    The first inequality is based on the fact that for a fixed $(o, a)$ pair, $\sum_{o'}\epsilon(o'|o, a) \leq 2$. On the other hand, by Cauchy-Shwarts inequality, we have:
    \begin{align*}
        (a) \leq \left[\sum_{o\in O^{\delta}, a} \left(\sum_{o'} \epsilon (o'|o, a)\right)^2\mu^{\pi}(o, a)\right]^{\frac{1}{2}} H^{\frac{1}{2}}
    \end{align*}

    By preconditions, for any $o \in \calO^{\delta}$ we always have
    \begin{align*}
        \max_{\Tilde{\pi}} \frac{\mu^{\Tilde{\pi}}(o, a)}{\lambda(o, a)} \leq 2OAH
    \end{align*}
    which leads to:
    \begin{align*}
        \sum_{o\in O^{\delta}, a} \left(\sum_{o'} \epsilon (o'|o, a)\right)^2\mu^{\pi}(o, a)\leq 2OAH \sum_{o\in O^{\delta}, a} \left(\sum_{o'} \epsilon (o'|o, a)\right)^2\lambda(o, a)
    \end{align*}
    By lemma \ref{lemma_rf_helper}, when $N \geq \frac{8O^3A^3H^2}{\delta^2}\ln{\left(\frac{2H}{\rho}\right)}$, then w.p. at least $1-2\rho$,
    \begin{align*}
        \sum_{o\in O^{\delta}, a} \left(\sum_{o'} \epsilon (o'|o, a)\right)^2\mu^{\pi}(o, a)\leq \frac{16O^3A^2H}{N}\ln{\left(\frac{2OAH}{\rho}\right)}
    \end{align*}
    Combining all equations above, we have
    \begin{align*}
        \left| \widehat{J}(\pi) - J(\pi) \right| \leq \Maxr\left(\sqrt{\frac{16O^3A^2H^2}{N}\ln{\left(\frac{2OAH}{\rho}\right)}} + 2O\delta\right)
    \end{align*}
    Choose $\delta = \frac{\epsilon}{8O\Maxr}$ and $N \geq \max{\left(\frac{512H^2O^3A^2\Maxr^2}{\epsilon^2}\ln{\left(\frac{2OAH}{\rho}\right)}, \frac{256O^5A^3H^2\Maxr^2}{\epsilon^2}\ln{\left(\frac{2H}{\rho}\right)}\right)}$,  the proof is completed. 
\end{proof}
\subsection{Proof of Lemma \ref{lem:final_difference}}
\begin{proof}
     We denote the optimal policy for NMRDP$(p, \calR)$ and empirical NMRDP$(\widehat{p}, \calR)$ as $\pi^*$ and $\widehat{\pi}^*$ respectively, then the following holds
     \begin{align*}
          \left| J(\pi^*)-J(\widehat{\pi})\right| & \leq \underbrace{|J(\pi^*) - \widehat{J}(\pi^*)|}_{\text{evaluation error}} + \underbrace{\widehat{J}(\pi^*) - \widehat{J}(\widehat{\pi}^*)}_{\leq 0\ \text{by definition}}
          + \underbrace{\widehat{J}(\widehat{\pi}^*) - \widehat{J}(\widehat{\pi})}_{\text{optimization error}} + \underbrace{|\widehat{J}(\widehat{\pi}) - J(\widehat{\pi})|}_{\text{evaluation error}}
     \end{align*}
     where evaluation errors are bounded by $\epsilon/2$ by lemma \ref{lemma_policy_diff}, and optimization error is $\alpha$ achieved by $\alpha$-approximate planner.
\end{proof}
\subsection{Proof of Theorem \ref{thm:reward_free}}\label{prf:thm_reward_free}
\begin{proof}
    Based on lemma \ref{lemma_policy_diff}, we need $\delta = \frac{\epsilon}{8O\Maxr}$, consequently, $N_0 \geq cO^3AH^4\Maxr\iota^3/\epsilon$. Since we need $N_0$ episodes for each $o$ observation, the total number of episodes of finding $\Psi$ is $\widetilde{O}(O^4AH^4\Maxr\iota^3/\epsilon)$, which gives second term of Theorem \ref{thm:reward_free}. We combine this with the result of lemma \ref{lemma_policy_diff}, which leads to the first term of Theorem \ref{thm:reward_free}. We complete our proof by considering the optimization error and policy evaluation error per lemma \ref{lem:final_difference}.
\end{proof}
\section{Experiment}\label{appendix:exp}
In this section, we provide the implementation details of our experiments.
\subsection{Doubling trick}\label{appendix:dt}
To speed up the computation, we apply the doubling trick originated from \citet{auer2008near}. Instead of updating empirical transition matrix every episode, we update it after certain amount of observations. Specifically, whenever there is a $(o,a)$ pair whose visitation counts reach the power of $2$, we start a new epoch, in which we recompute our empirical transition matrix $\widehat{p}$ (line 7 in algorithm \ref{alg_UCBVIRM}), empirical cross-product transition matrix $\widehat{P}$, rewards $\widehat{R}$(line 12 in algorithm \ref{alg_UCBVIRM}) and bonus function(line 13 in algorithm \ref{alg_UCBVIRM}). Then we calculate new $Q$ function (line 14 in algorithm \ref{alg_UCBVIRM}). Finally we execute policy according to new Q function until there is there is a $(o,a)$ pair whose visitation counts reach the power of $2$ to start a new epoch. This approach greatly reduce the computation and won't affect the $\widehat{O}$ statistical efficiency.
\subsection{Exploration coefficient}\label{appendix:ec}
\alg, \texttt{UCBVI}, \texttt{UCRL2-RM-L}, and \texttt{UCRL2-RM-B} all apply the principle of optimism in the face of uncertainty, where the algorithms either adjust the reward functions (\alg and \texttt{UCBVI}) or modify their models (\texttt{UCRL2-RM-L} and \texttt{UCRL2-RM-B}) to balance exploration and exploitation. Specifically, \alg and \texttt{UCBVI} carefully design exploration bonuses to ensure that $V_{k, h} \geq V_h^*$. In contrast, \texttt{UCRL2-RM-L} and \texttt{UCRL2-RM-B} construct a set of MDPs that likely contains the true MDP according to different concentration inequalities, then alter the model to be the best possible MDP within that set. However, due to the theoretical pessimism, these approaches often lead to over-exploration in practice, resulting in higher regret. To mitigate this, we tune the exploration coefficient of each algorithm to better balance exploration and exploitation in each environment, improving performance. For fairness, we select the optimal exploration coefficient for each algorithm from an equally large set of candidates.

Specifically, the exploration coefficient $\gamma$ of \alg is defined as the empirical bonus used in the experiments divided by the theoretical bonus function calculated using Algorithm \ref{alg_bonus}. This modifies line 13 of Algorithm \ref{alg_UCBVIRM} to be: $b_k(s, a) = \gamma \cdot \texttt{bonus}(\widehat{p}k(o, a), V{k, h+1}, N_k, N’_{k, h})$. \texttt{UCBVI} applies the same rule as \alg. For \texttt{UCRL2-RM-L}, the algorithm designs confidence sets for the transition function $p$ for every $(o, a)$ pair, such that the true dynamics lie within a confidence interval centered on the empirical mean $\widehat{p}$. Formally, $C_{(o,a)}=\{p’\in\Delta_o, \|p’-\widehat{p}(o,a)\|_1 \leq \gamma\cdot\beta(o, a)\}$,$\mathbbP[p\notin C_{(o,a)}]\leq \rho$, where $\gamma=1$ is the original parameter. The exploration coefficient for \texttt{UCRL2-RM-B} follows the same principle, with the only distinction being the confidence interval design. For more detailed implementation, please refer to our codes. For each algorithm, we choose parameters from an equally large set for each environment. Following is the table of candidates of exploration coefficient $\gamma$ for every algorithm,
\begin{table}[h!]
\centering
\label{tab:parameter_candidates}
\begin{tabular}{lccc}
\toprule
\textbf{Algorithm} &  \textbf{Candidates of $\gamma$} \\
\midrule
\alg &  $0.001$, $0.01$, $0.1$, $0.5$, $1$, $2$ \\
\texttt{UCBVI} &  $0.001$, $0.01$, $0.1$, $0.5$, $1$, $2$ \\
\texttt{UCRL2-RM-L} &  $0.1$, $0.25$,$0.5$, $0.75$, $1$, $2$ \\
\texttt{UCRL2-RM-B} & $0.01$, $0.1$, $0.5$, $0.75$, $1$, $2$ \\
\bottomrule
\end{tabular}
\caption{Exploration Coefficient Candidates for Three Algorithms}
\end{table}

According our results, all original algorithms over explore ($\gamma < 1$ leads to better performance) in all of our experiments. Surprisingly, we find a fixed set of parameters perform the best out of other opponents across all experiment settings. Specifically, we end up choosing $\gamma = 0.001$, $\gamma = 0.001$, $\gamma = 0.5$ and $\gamma = 0.1$ for \alg, \texttt{UCBVI}, \texttt{UCRL2-RM-L} and \texttt{UCRL2-RM-B}, respectively. Note that, with smaller $\gamma$, \texttt{UCRL2-RM-L} and \texttt{UCRL2-RM-B} will fail in converging in Extended Value Iteration (EVI) \citep{auer2008near} in all of our experiments.
\section{Technical Lemmas}\label{appendix:tl}
\begin{lemma}\citep{maurer2009empirical}
    Let $Z_1, \ldots, Z_n(n \geq 2)$ be i.i.d. random variables with values in $[0, 1]$ and let $\rho \in (0,1)$. Define $\Bar{Z} = \frac{1}{n}\sum_{i=1}^n Z_i$ and $\widehat{\mathbbV}_n = \frac{1}{n}\sum_{i=1}^n (Z_i-\Bar{Z})^2$. Then we have
    \begin{align*}
        \mathbbP\left[\left|\mathbbE[Z] - \Bar{Z}\right| > \sqrt{\frac{2\widehat{\mathbbV}_n\ln{(2/\rho)}}{n-1}} + \frac{7\ln{(2/\rho)}}{3(n-1)}\right] \leq \rho
    \end{align*}
\end{lemma}
\begin{lemma}\citep{cesa2006prediction}
    Let $Z_1, \ldots, Z_n(n \geq 2)$ be i.i.d. random variables with values in $[0, 1]$ and let $\rho \in (0,1)$. Define $\Bar{Z} = \frac{1}{n}\sum_{i=1}^n Z_i$ and $\widehat{\mathbbV}_n = \frac{1}{n}\sum_{i=1}^n (Z_i-\Bar{Z})^2$. Then we have
    \begin{align*}
        \mathbbP\left[\left|\mathbbE[Z] - \Bar{Z}\right| > \sqrt{\frac{2\widehat{\mathbbV}_n\ln{(2/\rho)}}{n}} + \frac{2\ln{(2/\rho)}}{3n}\right] \leq \rho
    \end{align*}
\end{lemma}
\begin{lemma}\citep{cesa2006prediction}
    Let $X_1, \ldots, X_n(n \geq 2)$ be a martingale difference sequence w.r.t. some filtration $\calF_n$, $\forall i\in [n], |X_i| \leq u$, then for any $n>1$, $u > 0$, $\rho \in (0,1)$ we have
    \begin{align*}
        \mathbbP\left[\left|\sum_{i=1}^n X_i\right| >\sqrt{2nu\ln{(2/\rho)}}\right] \leq \rho
    \end{align*}
\end{lemma}
\begin{lemma}\citep{freedman1975martingale}
    Let $X_1, \ldots, X_n(n \geq 2)$ be a martingale difference sequence w.r.t. some filtration $\calF_n$, $\forall i\in [n], |X_i| \leq u$. if the sum of the variances $\sum_{i=1}^n \Var(X_i|\calF_i) \leq w$ then for any $\rho \in (0,1)$, $n>1$, $u > 0$ and $w>0$ we have
    \begin{align*}
        \mathbbP\left[\sum_{i=1}^n X_i >\sqrt{2w\ln{(1/\rho)}} + \frac{2u\ln{(1/\rho)}}{3}\right] \leq \rho
    \end{align*}
\end{lemma}
\begin{lemma}\label{lemma_var_ineq}\citep{azar2017minimax}
    Let $X \in \mathbbR$ and $Y\in \mathbbR$ be two random variables. Then the following bound holds for their variances
    \begin{align*}
        \Var(X) \leq 2[\Var(Y)+\Var(X-Y)]
    \end{align*}
\end{lemma}

\end{document}